\pdfoutput=1
\documentclass[11pt]{article}

\usepackage[preprint]{acl}

\usepackage{times}
\usepackage{latexsym}

\usepackage[T1]{fontenc}
\usepackage[utf8]{inputenc}
\usepackage{microtype}
\usepackage{inconsolata}
\usepackage{graphicx}

\usepackage{amsmath}
\usepackage{amssymb}
\usepackage{amsthm}
\usepackage{booktabs}
\usepackage{tabularx}
\usepackage{makecell}
\usepackage{colortbl}
\usepackage{enumitem}
\usepackage{placeins}
\usepackage{csquotes}
\usepackage{pifont}
\usepackage{listings}

\lstdefinestyle{papercode}{
  basicstyle=\scriptsize\ttfamily,
  breaklines=true,
  breakatwhitespace=false,
  columns=fullflexible,
  keepspaces=true,
  showstringspaces=false,
  tabsize=2,
  frame=single,
  xleftmargin=0pt,
  xrightmargin=0pt,
  aboveskip=3pt,
  belowskip=3pt
}
\newtheorem{theorem}{Theorem}
\newtheorem{lemma}{Lemma}
\newtheorem{corollary}{Corollary}
\newcommand{\cmark}{\ding{51}}
\newcommand{\xmark}{\ding{55}}
\newcolumntype{Y}{>{\raggedright\arraybackslash}X}

\title{AMDKernelVault: Large-Scale Datasets and Agentic Training for \\ AMD GPU Kernel Optimization}

\author{
  Ji Liu$^{*,\dagger}$ \quad
  Saptarshi Majumder$^{*,\dagger}$ \quad
  Yiqing Huang$^{*,\dagger}$ \quad
  Wenwen Ouyang$^{*,\dagger}$ \\
  Umang Pandey \quad
  Zeping Li \quad
  Chushi Chen \quad
  Zihao An \\
  Puyuan Yang \quad
  Zekai Li \quad
  Sina Rafati \quad
  Ziqiong Liu \\
  Pratik Prabhanjan Brahma \quad
  Dong Li \quad
  Zicheng Liu \\
  Sharon Zhou \quad
  Emad Barsoum \\
  Advanced Micro Devices, Inc.\\
  \textsuperscript{*}\textit{Equal contribution.}\quad
  $^{\dagger}$\textit{Correspondence:} \texttt{\{liuji,sapmajum,Cody.Huang,viouyang\}@amd.com}.
}

\begin{document}
\maketitle
\begin{abstract}
We introduce AMDKernelVault, an open HIP and Triton kernel corpus and training framework for recent AMD CDNA GPUs. Existing LLM-based kernel agents are largely CUDA/NVIDIA-centric and often depend on repeated frontier-LLM calls for generation, reflection, and optimization. To address this gap, we develop HIPKernelGen and TritonKernelGen, agent-driven pipelines that transform PyTorch references into HIP or Triton kernels, compile and validate candidates under ROCm, and latency-profile them on AMD hardware. The corpus contains 62,153 execution-verified HIP kernel samples, 2,377 production-grounded ROCm Libraries QA entries, and 39,893 Triton kernels. We further train Qwen3-8B with supervised fine-tuning and execution-aware reinforcement learning as a demonstration of the corpus's utility. Under fixed evaluation budgets, it achieves the highest correctness among the compared models on PyTorch-to-HIP (34.0\% Pass@1), TritonBench-G (33.2\% Corr@3), and ROCmBench (41.94\% Corr@3), but does not uniformly lead compilation or speed metrics. The corpus and documentation are available at \url{https://huggingface.co/datasets/amd/AIG-Datasets}, and the associated training and kernel-generation code is available at \url{https://github.com/AMD-AGI/hip_kernel_llm_lab}.
\end{abstract}

\section{Introduction}
\label{sec:intro}

Recent advances in GPU kernel optimization have been driven by LLM-based~\citep{chen2025cudallm,baronio2026kevin} and agent-based approaches~\citep{wang2025geak,dai2026cuda}, together with standardized benchmarks such as KernelBench~\citep{pmlr-v267-ouyang25a,lange2025robustkbench,heakl2026cass}. These agent-based systems leverage compiler feedback, runtime correctness checks, and performance profiling to iteratively refine generated kernels~\citep{wang2025geak,chen2025cudallm}. However, the data and tooling behind these systems remain largely CUDA/NVIDIA-centric. In contrast, AMD's HIP/ROCm~\citep{amdrocmdoc,amdhipdoc} ecosystem lacks large-scale, execution-verified kernel data for training and evaluating kernel LLMs, especially across both native HIP and Triton backends.

This data gap is not a matter of syntax conversion alone. Although CUDA and HIP expose similar programming abstractions, performant AMD kernels depend on ROCm compilation behavior, wavefront execution, LDS usage, memory coalescing, matrix-core support, and Triton~\citep{tillet2019triton}  backend constraints. Direct translation through \texttt{hipify}~\citep{amdhipifydoc} may preserve API structure while retaining NVIDIA-shaped tiling constants, warp-size assumptions, input-specialized indexing, or memory layouts that fail to compile or underperform on AMD GPUs. Reliable AMD kernel data therefore requires native compilation, execution-based correctness validation, and latency measurement under ROCm.

We present \textbf{AMDKernelVault}, an open, execution-verified HIP and Triton~\citep{tillet2019triton} kernel corpus together with scalable agent-based pipelines for AMD GPU kernel data production. Our pipelines, HIPKernelGen and TritonKernelGen, start from PyTorch references and use a generate--evaluate--reflect loop to synthesize candidate kernels, compile them under ROCm, validate numerical correctness across test inputs, and profile latency on AMD hardware. This process produces 64K+ validated HIP/ROCm data and $\sim$40K Triton kernels for recent AMD CDNA GPUs, including benchmark-derived tasks, synthesized operator workloads, and production-grounded supervision from \texttt{rocBLAS} and \texttt{rocSOLVER}.

Beyond releasing data, we evaluate whether this corpus can train compact AMD kernel LLMs that reduce reliance on frontier LLMs in agentic optimization systems. Existing kernel-agent workflows often call frontier models repeatedly for generation, reflection, and optimization, which increases cost and limits deployment in private AMD environments. Using AMDKernelVault, we train Qwen3-8B~\citep{yang2025qwen3} with supervised fine-tuning and execution-aware reinforcement learning, then insert the trained model into a GEAK-style~\citep{wang2025geak} agent loop where one local policy serves as generator, reflector, and optimizer. The training recipe is intended as an effectiveness test for the corpus and pipeline rather than a new RL algorithm.

Under the same evaluation budget, the trained 8B model achieves the highest correctness among the compared models on PyTorch-to-HIP (34.0\% Pass@1) and TritonBench-G (33.2\% Corr@3), while running locally on AMD hardware. These results show that execution-verified AMD kernel data can support locally deployable kernel LLMs for ROCm environments. Our primary contribution is the resource and its AMD-native generation and validation infrastructure; the SFT+RL recipe is an effectiveness demonstration rather than a new learning algorithm. Comparisons with frontier models are specific to the reported correctness metrics, benchmarks, and agent budgets.

\paragraph{Our contributions:}
\begin{itemize}[leftmargin=*,itemsep=2pt]
\item \textbf{AMD-native kernel corpus.} We release 64K+ HIP/ROCm and $\sim$40K Triton data validated through ROCm compilation, execution-based correctness tests, latency profiling, or production-grounded HIP code supervision on recent AMD CDNA GPUs.

\item \textbf{Scalable data production pipelines.} HIPKernelGen and TritonKernelGen provide reusable agent-based generate--evaluate--reflect workflows for producing verified HIP and Triton kernel data from PyTorch references, including benchmark-derived, synthesized, and production-grounded sources.

\item \textbf{Compact AMD kernel LLMs.} As a corpus-utility demonstration, we train a local Qwen3-8B model to serve the generator, reflector, and optimizer roles in a GEAK-style loop, improving specific AMD HIP/Triton correctness metrics under the same evaluation budget.
\end{itemize}

\section{Related Work}
\label{sec:related}

\paragraph{GPU Kernel Optimization.}
GPU kernel optimization is a longstanding problem in high-performance ML systems~\citep{ragan2013halide,chen2018tvm,vasilache2018tensorcomprehensions}. Significant efforts have focused on operator fusion, tiling, memory hierarchy optimization, and custom kernel design. Representative systems such as FlashAttention~\citep{dao2022flashattention,shah2024flashattention} show that carefully designed kernels can provide substantial speedups for modern ML workloads. More recently, automated frameworks have used program analysis, search, and learning-based techniques~\citep{zheng2020ansor} to reduce the manual effort required for kernel tuning. However, most of these methods are developed and evaluated primarily on Nvidia GPUs and CUDA backends, leaving AMD-focused kernel optimization comparatively underexplored.

\begin{figure*}[t]
\centering
\includegraphics[width=0.90\textwidth]{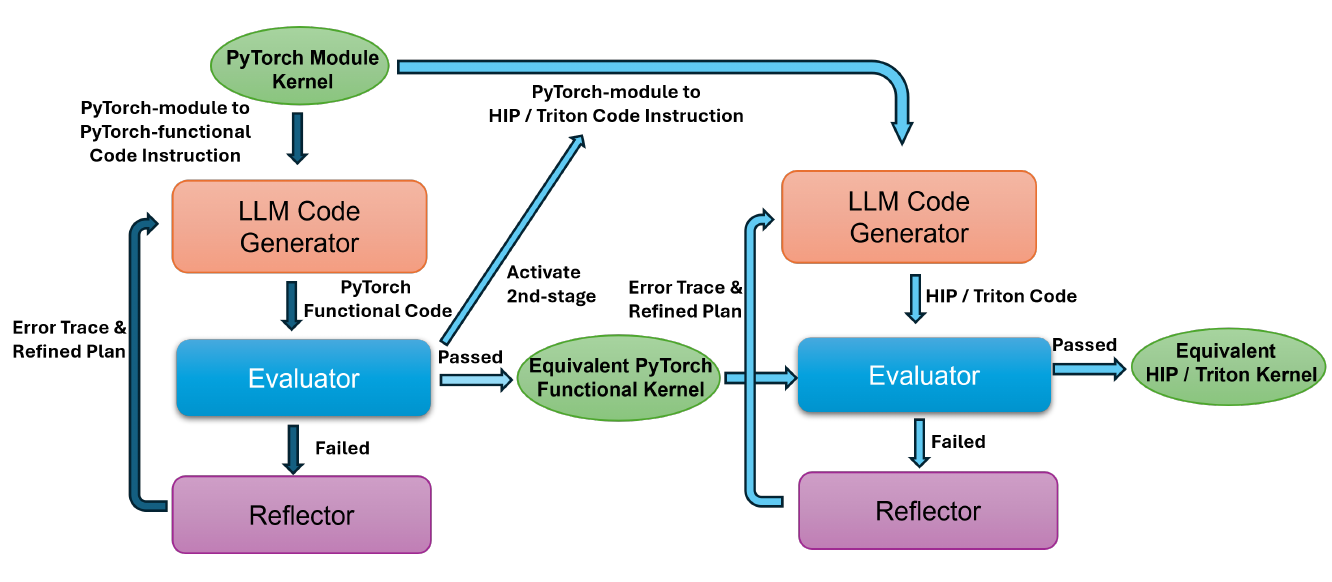}
\caption{\textbf{Unified kernel generation framework.} Stage 1 standardizes PyTorch modules into function-style references. Stage 2 synthesizes HIP and Triton kernels using a GEAK-based agentic loop with compilation, correctness validation, and latency-guided optimization.}
\label{fig:overview}
\end{figure*}

\paragraph{HIP, Triton, and ROCm.}
AMD's HIP provides a CUDA-like programming model within the ROCm software stack~\citep{amdrocmdoc,amdhipdoc}, while Triton offers a Python-based DSL for writing custom GPU kernels with explicit control over memory and parallelism~\citep{tillet2019triton}. Although these abstractions improve portability, AMD execution differs in ROCm compiler behavior, wavefront-level execution, LDS usage, memory coalescing, and backend scheduling constraints. Source-to-source tools such as \texttt{hipify} can preserve syntax but provide limited guarantees of correctness or performance portability~\citep{amdhipifydoc}. This motivates native HIP/ROCm validation rather than treating AMD support as a post-hoc translation step.

\paragraph{GPU Kernel Datasets and Benchmarks.}
High-quality datasets are essential for advancing kernel optimization research. KernelBench~\citep{pmlr-v267-ouyang25a} provides standardized environments for generating and evaluating optimized kernels, while other benchmarks cover specific operator families~\citep{che2009rodinia}. However, existing resources remain largely CUDA-centric. CASS~\citep{heakl2026cass} facilitates CUDA-to-HIP translation, but does not provide large-scale PyTorch-to-HIP, HIP-to-HIP, or AMD Triton data. Existing Triton benchmarks also primarily target NVIDIA assumptions that may not generalize to ROCm. AMDKernelVault addresses this gap by constructing a unified HIP/Triton corpus validated through compilation, execution-based correctness checks, and latency profiling on AMD GPUs.

\paragraph{LLM-Based Kernel Agents.}
Recent work explores LLM-based agents for kernel synthesis and optimization~\citep{chen2025cudallm,baronio2026kevin,li2025autotriton,wang2025geak}. Kevin~\citep{baronio2026kevin} demonstrates multi-turn RL for CUDA kernels, AutoTriton~\citep{li2025autotriton} and TritonRL~\citep{woo2025tritonrl} study execution-aware training for Triton generation, and GEAK~\citep{wang2025geak} provides an agentic framework for Triton optimization. AVO~\citep{chen2026avo} further explores long-horizon autonomous kernel evolution on NVIDIA GPUs. AMDKernelVault is complementary: it focuses on execution-verified AMD HIP/ROCm/Triton data and compact local policies for generate--evaluate--reflect loops. Additional discussion is provided in Appendix~\ref{app:more_related_work}.

\section{Agent-Based Dataset Generation}
\label{sec:datasets}

We present a unified framework for generating execution-verified GPU kernels across both HIP and Triton, using a shared generate--evaluate--reflect paradigm (Figure~\ref{fig:overview}).

\subsection{Unified Pipeline Architecture}

We propose a unified, agent-driven framework for large-scale GPU kernel data construction. The framework follows a shared \emph{generate--evaluate--reflect} paradigm, enabling systematic synthesis, validation, and refinement of GPU kernels grounded in standardized PyTorch references. Within this unified design, we instantiate two concrete pipelines—\textbf{HIPKernelGen} and \textbf{TritonKernelGen}—which specialize the same agent-driven workflow to the HIP and Triton ecosystems respectively.

\paragraph{Stage 1: PyTorch Reference Standardization.}
For each target operator or module kernel, we construct a function-style PyTorch reference derived from the PyTorch module input. This representation ensures reliable invocation, consistent interface semantics, and execution-based validation across both HIP and Triton backends. When available, compiler-generated kernels (e.g., Torch Inductor Triton code) are used for cross-checking but do not replace reference-based correctness.

\paragraph{Stage 2: Backend-Specific Kernel Synthesis.}
Given the standardized PyTorch reference, backend-specific agents generate candidate GPU kernels. Each candidate is subjected to compilation, execution-based correctness checks against the PyTorch reference, and latency profiling. The default HIP checker uses \texttt{torch.allclose} with \texttt{rtol=1e-4}, \texttt{atol=1e-4}, and \texttt{equal\_nan=True}. For Triton, task-specific tolerances may be configured from the reference task; otherwise, the fallback is \texttt{rtol=1e-4}, \texttt{atol=1e-3}, and \texttt{equal\_nan=False}. We use no fixed dtype- or operator-specific rule. Failures trigger reflector-guided regeneration until correctness criteria are met or the attempt limit is reached.

\paragraph{Test-case construction and validation.}
For PyTorch$\rightarrow$HIP and PyTorch$\rightarrow$Triton, test inputs are inherited from CUDA-Agent-Ops-6K~\citep{dai2026cuda}, GPUMode~\citep{kernelbook2025} and AI-CUDA-Engineer-Archive~\citep{lange2025aicudaengineer} when available and augmented with verified LLM-generated cases to cover additional shapes, strides, dtypes, and boundary conditions. HIP kernels, including HIP$\rightarrow$HIP optimization targets, are bound through \texttt{torch.utils.cpp\_extension} so that generated kernels can be executed under the same PyTorch reference tests. For Triton$\rightarrow$Triton, we extend TritonBench-style tests with ROCm-compatible cases and discard candidates that only pass input-specialized or shape-specialized checks. Thus, LLM-generated tests are not accepted as supervision by themselves; they are used as additional probes inside an execution-based validation loop.

\paragraph{HIPKernelGen.}
HIPKernelGen instantiates the unified framework for HIP kernels targeting AMD GPUs. Much of the HIP corpus was generated from August to December 2025 using GPT-5, a strong available kernel-generation model during collection. Each candidate is compiled via \texttt{hipcc}, validated against PyTorch references, and profiled for latency and speedup. Retention depends on ROCm execution rather than generator identity.

\paragraph{TritonKernelGen.}
TritonKernelGen mirrors the HIP pipeline for Triton kernels. We use GPT-oss-120B~\citep{agarwal2025gpt}, DeepSeek-R1~\citep{guo2025deepseek}, and Qwen2.5-32B~\citep{qwen2025qwen25technicalreport} to diversify implementations and reasoning traces. Each kernel is compiled and executed via Triton 3.3.0 and tested on approximately ten correctness and three performance configurations spanning shapes, strides, FP32/FP16/BF16 where applicable, and boundary cases.

\subsection{Source Data}

Our kernel datasets are built upon six complementary sources spanning synthesized CUDA-agent tasks, GitHub-derived PyTorch programs, curated Triton corpora, benchmark-style kernels, and production ROCm libraries.

\subsubsection{CUDA-Agent-Ops-6K Dataset}
The \textit{CUDA-Agent-Ops-6K} dataset~\citep{dai2026cuda} contains 6,000 synthesized operator-level PyTorch tasks for CUDA kernel generation and optimization. It is built from \texttt{torch}/\texttt{transformers} seed operators, LLM-composed multi-operator tasks, and execution-based filtering. We use it as a source for HIP generation, but the original dataset remains CUDA-centric and does not directly support HIP or ROCm.

\subsubsection{GPUMODE-KernelBook Dataset}
The \textit{GPUMODE-KernelBook}~\citep{kernelbook2025} dataset contains approximately 18,000 PyTorch modules paired with Torch Inductor Triton implementations from GitHub-derived workloads. It provides broad operator coverage across primitives, fused kernels, and larger model components, but remains CUDA/Triton-centric and does not directly provide AMD-validated HIP kernels.

\subsubsection{AI-CUDA-Engineer-Archive Dataset}
The \textit{AI-CUDA-Engineer-Archive}~\citep{lange2025aicudaengineer} dataset is built with the \textit{AI CUDA Engineer} framework from over 250 KernelBench PyTorch benchmarks. It provides PyTorch references, generated CUDA kernels, diagnostics, and profiling results, but remains CUDA-centric and does not directly support HIP.

\subsubsection{Stack-v2-dedup-Triton Dataset}
We incorporate the curated \textit{Stack-v2-dedup-Triton}~\citep{lozhkov2024starcoder} dataset, constructed by filtering approximately 13,000 candidate samples to obtain 2,269 high-quality Triton--PyTorch pairs. This source provides medium-difficulty Triton examples emphasizing correctness and moderate complexity.

\subsubsection{TritonBench-8k Dataset}
The \textit{TritonBench-8k}~\citep{li2025tritonbench} dataset consists of samples derived from web crawl data and synthetic code generated via the Ninetoothed DSL~\citep{huang2025ninetoothedtritonbasedhighleveldomainspecific}. We apply targeted curation to remove DSL-specific artifacts and ensure correctness across diverse shapes, strides, and precision types.

\subsubsection{ROCm Libraries Dataset Source}
As a production-oriented source, we leverage AMD's ROCm Libraries monorepo, focusing on \texttt{rocBLAS} and \texttt{rocSOLVER}. Unlike benchmark-style datasets, these libraries capture engineering-tested HIP kernels, launch-side logic, and idiomatic production patterns.

\subsection{Generated Kernel Corpus}

Table~\ref{tab:datasets} summarizes our released corpus. HIP and Triton kernel entries are processed through HIPKernelGen/TritonKernelGen with execution-based validation on AMD GPUs. The ROCm Libraries QA subset provides complementary production-grounded supervision rather than generated kernel-pair data.

\begin{table}[t]
\centering
\small

\begin{tabular}{lr}
\toprule
\textbf{Dataset} & \textbf{Samples} \\
\midrule
\multicolumn{2}{l}{\textit{HIP/ROCm Datasets}} \\
HIP-CudaAgent & 5,388 \\
HIP-GPUMode & 22,397 \\
HIP2HIP (Optimization) & 34,368  \\
ROCm Libraries QA & 2,377 \\
\cmidrule{2-2}
\textit{Total HIP/ROCm} & \textit{64,530} \\
\midrule
\multicolumn{2}{l}{\textit{Triton Datasets}} \\
Triton-Stack  & 2,269 \\
Triton-Bench (from TritonBench-8k)  & 7,713 \\
Triton-GPUMode & 18,000 \\
Triton-AICE  & 11,911 \\
\cmidrule{2-2}
\textit{Total Triton} & \textit{39,893} \\
\bottomrule
\end{tabular}
\caption{Summary of the corpus used in this paper. HIP and Triton kernel samples are execution-validated on AMD GPUs; ROCm Libraries QA provides production-grounded supervision.}
\label{tab:datasets}
\end{table}

\paragraph{Counting and audit.}
The 64,530 HIP/ROCm total comprises 62,153 execution-verified HIP kernel samples and 2,377 ROCm Libraries QA entries. It excludes a separately released 14,282-entry HIP-AICUDA subset that was not used in this paper. HIP-CudaAgent retains 5,388 of 6,000 upstream tasks after excluding 612 incomplete records; HIP-GPUMode contains 5,910 unique tasks and 22,397 kernel variants; and HIP2HIP contains 31,155 single-turn and 3,213 multi-turn records. Triton counts use 39,893 unique source-level samples; packaged reasoning variants may yield a different row count. Since rejection logs were not aggregated uniformly across source pipelines, we report retained counts and subset-specific validation rather than a single corpus-wide failure rate.
\paragraph{HIP Datasets.}
\textbf{HIP-CudaAgent}: Built from \textit{CUDA-Agent-Ops-6K} using HIPKernelGen. We adapt the synthesized PyTorch operator tasks into function-style references and use the generate--evaluate--reflect pipeline to produce HIP implementations. Each candidate is compiled with the ROCm toolchain, checked against the PyTorch reference, and filtered by execution-based correctness validation. This yields 5,388 validated PyTorch-to-HIP samples.

\textbf{HIP-GPUMode}: Generated from GPUMODE-KernelBook using HIPKernelGen. We first convert PyTorch modules to function-style references, then generate HIP kernels using frontier LLMs. Each candidate undergoes compilation, execution-based validation, and latency profiling. The generate--evaluate--reflect loop iterates until correctness is achieved or a maximum attempt limit is reached. This produces 5,910 unique PyTorch entries with 22,397 HIP kernel variants (avg 3.8 variants/entry).

\textbf{HIP2HIP}: We extend HIPKernelGen to kernel-to-kernel optimization by operating directly on existing HIP kernel functions or operators. Using HIP-GPUMode as a base, we generate 34,368 validated HIP kernel optimization pairs.

\paragraph{ROCm Libraries QA.}
The ROCm Libraries QA subset contains production-grounded interface-level supervision rather than generated kernel-pair data. For each interface in \texttt{rocBLAS} and \texttt{rocSOLVER}, we identify kernel-relevant code regions and generate QA pairs targeting computation semantics, index-to-tensor mappings, memory access patterns, synchronization strategies, and edge cases. This yields 1,858 entries from \texttt{rocBLAS} and 519 from \texttt{rocSOLVER} (\textbf{2,377} total), capturing production HIP development patterns.

\paragraph{Triton Datasets.}
TritonKernelGen processes all Triton sources through a unified validation pipeline, validating numerical correctness across diverse tensor shapes, strides, and dtypes (FP32/FP16/BF16) on AMD GPUs.

\textbf{Triton-Stack}: Curated from Stack-v2-dedup, filtering $\sim$13K candidates to 2,269 medium-difficulty Triton--PyTorch pairs with verified AMD compatibility, then generated and execution-verified with TritonKernelGen.

\textbf{Triton-Bench}: Derived from TritonBench-8k, removing DSL-specific artifacts and validating 7,713 samples, generated and execution-verified with TritonKernelGen.

\textbf{Triton-GPUMode}: Sourced from GPUMODE-KernelBook ($\sim$18K entries), with Triton implementations ported from torch-compile kernels to generic kernels able to handle all shapes and params, generated and execution-verified with TritonKernelGen.

\textbf{Triton-AICE}: From AI-CUDA-Engineer-Archive (11,911 samples), converting CUDA-centric Triton to AMD-validated implementations with TritonKernelGen.

Together, this yields 62,153 validated HIP kernel samples and 39,893 validated Triton--PyTorch samples for AMD GPUs, plus 2,377 production-grounded ROCm Libraries QA entries. Additional details on dataset construction, source-specific processing, metadata, and difficulty distributions are provided in Appendix~\ref{app:datasets} and Appendix~\ref{app:analysis}.

\subsection{Production-Grounded ROCm QA Supervision}

To extend GPU kernel data construction beyond benchmark-style kernels, we extract HIP-kernel-focused supervision from production-grade ROCm libraries at the interface level. This subset is not treated as direct execution-verified kernel-pair data; instead, it provides production-grounded QA supervision anchored to retrieved \texttt{rocBLAS} and \texttt{rocSOLVER} code.

\paragraph{Interface decomposition.}
For each target project, we start from a user-facing interface and decompose it into a compact kernel-relevant scope. We traverse the call graph from the interface entrypoint to identify code regions directly involved in GPU execution, including device-side kernels, host-side launch logic, parameter preparation, and tightly coupled helper utilities. This interface-level view preserves the connection between public APIs and the underlying HIP kernels, which is often lost when treating files or repositories as unstructured text.

\paragraph{Scope control.}
Production libraries contain substantial infrastructure code unrelated to kernel behavior. To avoid recursion blow-up and irrelevant context, we use project-specific allowlists and disallowlists. The disallowlist filters logging, debugging, testing, and auxiliary utilities, while the allowlist constrains traversal to kernel-relevant subsystems. Because \texttt{rocBLAS} and \texttt{rocSOLVER} differ in organization and dispatch structure, each project uses tailored traversal rules.

\paragraph{Kernel-centric QA generation.}
Given the retrieved interface context, an agent constructs HIP-kernel-focused QA pairs. The questions target concrete implementation aspects, including computation semantics, index-to-tensor mappings, memory access patterns, synchronization and reduction strategies, launch configuration, and correctness constraints. Answers are grounded in the retrieved code context rather than generic library descriptions, so the resulting supervision emphasizes actionable kernel logic.

Applying this pipeline yields 519 interface-level entries from \texttt{rocSOLVER}, comprising 226 kernel-implementation-focused samples and 293 kernel-centric QA explanations. From \texttt{rocBLAS}, we extract 1,858 entries, including 1,169 kernel-implementation samples and 689 QA-style explanations. Together, these 2,377 production-grounded entries complement benchmark-derived HIP and Triton datasets with real-world host--device orchestration and production HIP development patterns.

\subsection{Difficulty Distribution}

We classify samples into three difficulty levels for curriculum training: Level 1 covers standalone single-function kernels, Level 2 covers fused operators or single-file implementations, and Level 3 covers multi-file kernels with cross-module dependencies. The full Triton corpus comprises 11,824 L1 (29.6\%), 24,765 L2 (62.1\%), and 3,304 L3 (8.3\%) samples; the curriculum-eligible subset used per Triton RL phase is reported in Appendix~\ref{app:training}. The ROCm Libraries QA data follows a similar distribution: \texttt{rocSOLVER} (190 L1, 198 L2, 131 L3) and \texttt{rocBLAS} (971 L1, 506 L2, 381 L3).

\section{Training Compact AMD Kernel LLMs}
\label{sec:training}

AMDKernelVault is designed not only as a kernel corpus, but also as a training resource for compact AMD kernel LLMs. We use it to train Qwen3-8B with a two-stage recipe: supervised fine-tuning (SFT) for cold-start kernel generation, followed by execution-aware reinforcement learning (RL) that uses compilation, correctness, and latency feedback.

\subsection{Kernel Tasks and Metrics}

Our datasets support four kernel-centric tasks: \emph{PyTorch$\rightarrow$HIP} translates PyTorch modules into semantically equivalent HIP implementations; \emph{HIP$\rightarrow$HIP} optimizes existing HIP kernels; \emph{Text$\rightarrow$Triton} generates Triton kernels from natural-language specifications; and \emph{Triton$\rightarrow$Triton} improves Triton kernels through tiling, memory-access, and boundary-handling changes.

We evaluate kernels with three execution-based metrics: \emph{compilation pass}, successful build under the target toolchain; \emph{correctness pass}, numerical equivalence against the PyTorch reference; and \emph{speed}, kernel-level speedup over the input kernel or PyTorch baseline. Comp@$k$ and Corr@$k$ report compilation and correctness after $k$ agent iterations, while Pass@$k$ measures the probability of obtaining at least one correct result from $k$ samples. For PyTorch-to-Triton, speedup is measured against the original PyTorch implementation; for Triton-to-Triton, it is measured against the input kernel. Speed@10 averages the speedups of the best ten correct samples.

\subsection{Supervised Fine-Tuning}

Direct RL from a pretrained model is difficult in this domain because early rollouts often fail to compile or violate kernel interfaces, producing sparse feedback. We therefore first apply SFT to teach kernel syntax, binding conventions, memory-access idioms, and common implementation patterns.

For HIP, we use instruction--response pairs from the HIP portions of AMDKernelVault, where the input is a PyTorch task or source HIP kernel and the output is a validated HIP implementation. For Triton, we use reasoning-augmented examples when available, so that the model learns to plan tiling, masking, boundary checks, and memory layouts before emitting code. This stage provides the cold-start capability needed for subsequent execution-aware RL.

\subsection{Execution-Aware Reinforcement Learning}

After SFT, we optimize the model using execution feedback. HIP training uses single-turn GRPO, where each candidate is compiled, tested, and scored after one generation attempt. Triton training uses a multi-turn variant aligned with the GEAK-style generate--evaluate--reflect loop: the model proposes a kernel, observes compiler or runtime feedback, reflects on the failure or performance bottleneck, and generates an improved candidate.

The reward combines three signals: compilation, correctness, and speedup. Compilation credit provides partial feedback for syntactically and interface-compatible kernels, correctness rewards functional equivalence, and a bounded speedup term encourages performance improvements without allowing timing outliers to dominate learning. This reward is designed for the low-success regime of AMD kernel generation, where binary correctness rewards alone provide too little gradient signal. Full reward definitions and GRPO details are provided in Appendix~\ref{app:training_method_details}.

\subsection{Agentic Integration}

A key goal is to test whether the corpus can train a compact local model for an agentic kernel optimization loop. We train one Qwen3-8B policy to serve the generator, reflector, and optimizer roles. During multi-turn Triton RL, the same model produces a candidate, interprets execution feedback, and emits the next revision.

This training recipe is not presented as a new RL algorithm. Rather, it evaluates whether execution-verified AMDKernelVault data can specialize a compact policy for AMD kernel agents. Section~\ref{sec:experiments} therefore reports controlled, metric-specific comparisons under identical agent budgets.

\subsection{Training Configuration}

We train Qwen3-8B with SFT followed by RL on AMD MI-series GPUs. HIP training uses SFT followed by single-turn GRPO, while Triton training uses SFT followed by multi-turn on-policy reflection RL with curriculum scheduling. We keep the main text focused on the training recipe used to evaluate AMDKernelVault as a corpus for compact AMD kernel LLMs. Appendix~\ref{app:training_method_details} provides task serialization, SFT and GRPO objectives, reward implementation, curriculum schedule, hyperparameters, and additional analysis of reward sensitivity.
\section{Experiments}
\label{sec:experiments}

We evaluate whether AMDKernelVault can train compact AMD kernel LLMs that improve HIP and Triton kernel generation under execution-based validation. We compare Qwen3-8B trained on our corpus against frontier LLMs (GPT-5, Gemini 2.5 Pro, Claude Sonnet 4). For HIP, we evaluate PyTorch-to-HIP translation on 200 KernelBench Level-1/2 tasks~\citep{pmlr-v267-ouyang25a} and HIP-to-HIP optimization on 100 held-out samples. For Triton, we use TritonBench-G (184 kernels) and ROCmBench (31 AMD-specific kernels) with the same 3-iteration OptimAgent-v2 protocol~\citep{wang2025optimv2}. Each model serves as generator, reflector, and optimizer under identical agent settings, and all reported Triton inference runs use MI355X GPUs.

\paragraph{Train/test separation and overlap controls.}
Evaluation cases are excluded from training by source and task identifier. HIP evaluation uses 200 held-out KernelBench-derived translation tasks and 100 held-out HIP-to-HIP optimization tasks; Triton evaluation uses the curated 184-kernel TritonBench-G and 31-kernel ROCmBench sets. These controls prevent exact task-identifier reuse. However, because generated kernels can implement semantically similar operator patterns using different code, exact-match filtering cannot exclude every functional near-duplicate. We did not conduct an exhaustive corpus-wide semantic-overlap search and therefore interpret the reported results as benchmark-specific rather than as evidence of overlap-free generalization.

\paragraph{Evaluation protocol and metrics.}
All Triton model comparisons use the same OptimAgent-v2 configuration: \texttt{num\_offsprings=1}, \texttt{max\_iteration=3}, identical generator/reflector/optimizer role assignment, and no profiler feedback. Thus, each row changes the model checkpoint while holding the external agent scaffold and execution budget fixed. Comp@$k$ and Corr@$k$ report the fractions of tasks that compile and pass numerical validation by iteration $k$, respectively; Pass@$k$ reports the probability of obtaining at least one correct result from $k$ independently sampled candidates. For PyTorch-to-kernel generation, speedup is measured against the PyTorch reference; for kernel-to-kernel optimization, it is measured against the input kernel. Speed comparisons include only correctness-passing candidates.

\paragraph{Benchmark coverage.}
The evaluation covers attention and softmax variants, layer and RMS normalization, GEMM-family operations, elementwise transformations, reductions, fused blocks, MoE-style routing, and online-softmax-style fused reductions. KernelBench Level-1/2 tests emphasize operator implementation and composition, while TritonBench-G spans difficulty levels from standalone operators to harder fused kernels; ROCmBench adds AMD-specific cases. This breadth tests multiple common deep-learning kernel patterns, but it is not a per-family performance study: the benchmark sizes are limited, and we do not report operator-family-specific pass rates.

\paragraph{Triton reproducibility setup.}
Triton workflows use Triton 3.3.0. Corpus validation and profiling use MI325 and MI350 GPUs; Triton SFT and RL use 32 MI325X GPUs; and reported Triton inference uses MI355X GPUs. The corpus targets gfx942 for MI325X stages and gfx950 for MI350/MI355X stages. ROCm, PyTorch, and Python versions vary across data-production snapshots and are not pinned to one release. Warm-up and measurement use 25\,ms and 100\,ms time budgets, respectively, with per-kernel iteration counts derived from a short timing estimate. We report median GPU-event timing, flush L2 before each measured call, synchronize after the measured loop, and discard no timing observations.

\subsection{HIP Kernel Results}

Table~\ref{tab:hip} evaluates PyTorch-to-HIP translation. Frontier LLMs compile many candidates but achieve limited correctness; GPT-5 reaches 75.5\% compilation but only 24.0\% correctness. Training on AMDKernelVault substantially improves Qwen3-8B: SFT raises correctness to 31.5\%, and RL further improves it to 34.0\%, outperforming all frontier baselines in correctness.

\begin{table}[t]
\centering
\small
\resizebox{\columnwidth}{!}{
\begin{tabular}{lcc}
\toprule
\textbf{Model} & \textbf{Comp Acc} & \textbf{Corr Acc} \\
\midrule

Gemini 2.5 Pro & 70.5\%& 18.5\% \\
Claude Sonnet 4 & \textbf{85.0\%} & 22.0\% \\
GPT-5 &  75.5\%& 24.0\% \\
\midrule
Q3-8B (base) & 50.0\% & 14.5\% \\
Q3-8B + SFT & 78.0\% & 31.5\% \\
\rowcolor{green!10} \textbf{Q3-8B + SFT + RL} & 83.5\% & \textbf{34.0\%} \\
\bottomrule
\end{tabular}
}
\caption{Pass@1 results of PyTorch-to-HIP translation (200 tasks from KernelBench Levels 1 and 2).  Q3-8B denotes Qwen3-8B.}
\label{tab:hip}
\end{table}

Table~\ref{tab:hip2hip} evaluates HIP-to-HIP optimization. The trained model reaches 84\% correctness, close to GPT-5 (87\%) and substantially above the base model (18\%), while maintaining positive speedup. These results indicate that the HIP portion of AMDKernelVault supports both translation and optimization tasks.

\begin{table}[t]
\centering
\small
\setlength{\tabcolsep}{3pt}
\resizebox{\columnwidth}{!}{
\begin{tabular}{lccc}
\toprule
\textbf{Model} & \textbf{Comp Acc} & \textbf{Corr Acc} & \textbf{Speed@10}\\
\midrule

Gemini 2.5 Pro & 81\% & 69\% & 1.15$\times$ \\
Claude Sonnet 4 & \textbf{93\%} & 67\% & 1.11$\times$ \\
GPT-5 & 89\% & \textbf{87\%} &\textbf{1.32$\times$} \\
\midrule
Q3-8B (base) & 26\% & 18\%  &1.02$\times$\\
Q3-8B + SFT & 86\% & 70\%  & 1.06$\times$\\
\rowcolor{green!10} \textbf{Q3-8B + SFT + RL} & 90\% & 84\% & 1.14$\times$ \\
\bottomrule
\end{tabular}
}
\caption{Pass@1 results of HIP-to-HIP kernel optimization (100 held-out modules). Q3-8B denotes Qwen3-8B and speed@10 denotes average speedup achieved
over baseline of best 10 correct samples.}
\label{tab:hip2hip}
\end{table}

\subsection{Triton Kernel Results}

Table~\ref{tab:triton_combined} reports TritonBench-G results under the fixed 3-iteration GEAK OptimAgent-v2 evaluator. The base Qwen3-8B model reaches 4.9\% Corr@3, while SFT improves to 10.3\% and SFT+RL reaches 33.2\%, compared with GPT-5 at 15.2\% and Claude Sonnet 4 at 29.8\%. Since all rows use the same scaffold and budget, this controlled comparison demonstrates the corpus's utility for specializing a compact local model.

\begin{table}[h]
\centering

\footnotesize
\setlength{\tabcolsep}{3pt}
\resizebox{\columnwidth}{!}{
\begin{tabular}{lcccc}
\toprule
\textbf{Model} & \textbf{Corr@1} & \textbf{Corr@3} & \textbf{Comp@3} & \textbf{Speed} \\
\midrule
Gemini 2.5 Pro & 8.7\% & 23.4\% & 32.6\% & 1.31$\times$ \\
Claude Sonnet 4 & 7.6\% & 29.8\% & 46.7\% & 1.34$\times$ \\

GPT-5 & 3.8\% & 15.2\% & 28.8\% & 1.05$\times$ \\
\midrule
Q3-8B (base) & 0.0\% & 4.9\% & 14.7\% & 0.97$\times$ \\
Q3-8B + SFT & 4.3\% & 10.3\% & 27.7\% & 1.19$\times$ \\
\rowcolor{green!10} \textbf{Q3-8B + SFT + RL} & \textbf{12.0\%} & \textbf{33.2\%} & \textbf{67.9\%} & \textbf{1.46$\times$} \\
\bottomrule
\end{tabular}
}
\caption{TritonBench-G results under the fixed 3-iteration GEAK OptimAgent-v2 evaluator. All models use the same agent budget and role assignment; only the LLM checkpoint changes. Corr/Comp denote correctness/compile accuracy, and Q3-8B denotes Qwen3-8B.}
\label{tab:triton_combined}
\end{table}

Table~\ref{tab:rocmbench} shows results on ROCmBench. The trained Qwen3-8B model obtains the highest Corr@3 (41.94\%), compared with Claude Sonnet 4 (35.48\%) and GPT-5 (29.03\%). Claude Sonnet 4 remains stronger on Comp@3 (67.74\% versus 58.06\%) and speed (1.82$\times$ versus 1.61$\times$), so the advantage is specific to correctness.

\paragraph{Difficulty Analysis.}
Our model leads the compared frontier baselines on D1--D3 of TritonBench-G, but Claude Sonnet 4 solves more D4 kernels (18/84 versus 13/84), and D5 remains unsolved by all models. Full results are provided in Appendix~\ref{app:results}.

\begin{table}[t]
\centering

\footnotesize
\setlength{\tabcolsep}{3pt}
\resizebox{\columnwidth}{!}{
\begin{tabular}{lcccc}
\toprule
\textbf{Model} & \textbf{Corr@1} & \textbf{Corr@3} & \textbf{Comp@3} & \textbf{Speed} \\

\midrule
Gemini 2.5 Pro & 12.90\% & 25.81\% & 41.94\% & 1.47$\times$ \\
Claude Sonnet 4 & 19.35\% & 35.48\% & \textbf{67.74\%} & \textbf{1.82$\times$} \\
GPT-5 & 16.13\% & 29.03\% & 35.48\% & 1.31$\times$ \\
Q3-8B (base) &  3.22\%  &  6.45\% & 19.35\% & 0.74$\times$ \\
\midrule
\rowcolor{green!10} \textbf{Q3-8B + SFT + RL} & \textbf{22.58\%} & \textbf{41.94\%} & 58.06\% & 1.61$\times$ \\
\bottomrule
\end{tabular}
}
\caption{ROCmBench (AMD) results under the fixed 3-iteration GEAK OptimAgent-v2 evaluator. All rows use the same agent budget, and Q3-8B denotes Qwen3-8B.}

\label{tab:rocmbench}
\end{table}

\paragraph{Limitations and Failure Analysis.}
Despite strong results on D1--D3 tasks, expert-level kernels remain challenging. Common failure modes include incorrect tile/block size choices for unusual tensor shapes, synchronization errors in multi-stage reductions, and memory-coalescing failures under AMD's wavefront execution model. These failures suggest that future work should incorporate richer profiler-guided feedback and broader expert-kernel supervision.

\subsection{Ablation Studies}

We conduct ablations to assess the importance of data scale, SFT cold start, and multi-turn RL. Performance improves with more training data, SFT provides a critical initialization before RL, and multi-turn RL contributes +13.1\% on Triton. Detailed ablation tables and reward-sensitivity results are provided in Appendix~\ref{app:ablation}.

\section{Conclusion}
\label{sec:conclusion}

AMDKernelVault is an open HIP and Triton kernel resource for recent AMD CDNA GPUs, supported by AMD-native generation and validation infrastructure. As a utility demonstration, we specialize Qwen3-8B for a GEAK-style loop; under fixed budgets, it achieves the highest correctness on selected benchmarks but does not uniformly lead compilation or speed. We release the corpus, documentation, and code to support future AMD kernel research.

\section*{Limitations}

Our evaluation is limited to AMD GPU environments and kernel-generation benchmarks available to us. Expert-level kernels remain difficult across all tested models, and the reported performance depends on the specific ROCm, Triton, and hardware configurations used during validation. Although we include production-derived ROCm library supervision, some generated kernels may still inherit suboptimal structures from upstream CUDA/Triton sources or from LLM-generated candidates.

Several gaps remain for future work: broader corpus-level profiling of latency, memory footprint, register pressure, and occupancy; per-operator-family pass-rate analysis beyond the current difficulty breakdown; larger held-out HIP evaluation sets; and evaluation beyond deep-learning workloads, including HPC, sparse/graph, and database kernels. Future work should also incorporate richer profiler-guided feedback and broader expert-kernel supervision.

We do not report a corpus-wide semantic near-duplicate audit or uniformly aggregated candidate-rejection rates because source pipelines use different logging formats. Software environments also vary across data-production snapshots, although Triton is fixed to version 3.3.0 for the reported Triton workflow.

\section*{Ethical Considerations and Impact Statement}

This work focuses on GPU kernel generation and optimization for AMD hardware. The released artifacts consist of program code, generated kernels, execution metadata, and QA supervision grounded in public ROCm library code. Provenance, usage conditions, and per-source license notes are included with the public artifact.

The expected positive impact is improved efficiency for AI workloads. Faster generated kernels can reduce training and inference runtime, while a compact local model can reduce repeated frontier-LLM calls during kernel search. These improvements may lower computational cost and energy use for AMD GPU deployments.

As with other code-generation systems, generated code may be incorrect, inefficient, or unsafe if deployed without validation. We mitigate this risk by emphasizing execution-based compilation, correctness testing, and latency profiling throughout data construction and evaluation. The work is narrowly focused on performance engineering for GPU kernels, and generated kernels should be validated in the target environment before production use.

\bibliography{custom}

\clearpage
\appendix

\newpage

\section{Artifact, License, and AI-Assistance Statement}
\label{app:artifact_statement}

\paragraph{Artifacts and licenses.}
The data and documentation are available at
\url{https://huggingface.co/datasets/amd/AIG-Datasets}.
The associated training, RL, and HIP kernel-generation code is available at
\url{https://github.com/AMD-AGI/hip_kernel_llm_lab}.
The release includes corpus-audit, validation-protocol, and
license/provenance documentation. Upstream resources remain subject
to their source-specific terms: CUDA-Agent-Ops-6K and
AI-CUDA-Engineer-Archive are released under CC-BY-4.0;
the current GPUMODE/KernelBook release uses the June 9 Researcher
Reciprocity License; TritonBench-8k uses Apache-2.0; and Stack v2
retains the BigCode and original-repository terms. ROCm Libraries
are licensed per component; the source components used here include
rocBLAS under the MIT License and rocSOLVER under the BSD-2-Clause
License. The ROCm Libraries QA subset contains interface-level
supervision and provenance references rather than a redistribution
of the complete source trees. Exact source revisions and applicable
terms are documented in the released provenance metadata.

\paragraph{Data content and filtering.}
The released data consists of program code, kernel implementations, task metadata, execution logs, and interface-level QA supervision grounded in public ROCm library code. It is not intended to contain personally identifying information or offensive content, and the data construction pipeline filters generated artifacts for executability, determinism, and task validity. Generated kernels are filtered through compilation, execution-based correctness checks, and ROCm validation; QA entries are grounded in retrieved public code contexts rather than user data.

\paragraph{AI assistance.}
We used AI assistants during this project for two purposes. First, frontier LLMs such as GPT-5 were used inside HIPKernelGen/TritonKernelGen to generate candidate GPU kernels and reflection feedback, which were then compiled, tested, profiled, and filtered by the authors' execution-based pipeline. Second, general AI writing assistants were used to help polish wording and organize the manuscript. All experimental design, dataset filtering, validation, analysis, and final claims were checked and approved by the authors.

\section{More Related Work}
\label{app:more_related_work}
\paragraph{Production Library Code as Supervision.}
Benchmark-derived datasets provide diversity but are typically limited to standalone kernels. In contrast, production library code from projects such as \texttt{rocBLAS} and \texttt{rocSOLVER} captures interface-level implementations spanning multiple files and reflecting production-grade engineering practices. These libraries encode years of architecture-specific optimization and correctness engineering. AMDKernelVault incorporates this signal through an interface-aware extraction pipeline that produces kernel-centric supervision grounded in real HIP code.
\paragraph{Positioning Relative to Multi-Turn RL Kernel Agents.}
Prior multi-turn systems, including Kevin~\citep{baronio2026kevin}, demonstrate that iterative execution feedback can improve kernel generation. We do not claim that multi-turn feedback alone is new. Our setting differs in two ways. First, AMD HIP/Triton generation starts from a lower-success regime, where early compilation and correctness rewards are sparse. Second, our goal is to train a model that can be deployed inside an agent-training framework, rather than only evaluate an external LLM called by an agent wrapper. In GEAK-style workflows, repeated LLM calls normally handle generation, reflection, and optimization. We instead train one compact policy on execution-feedback trajectories so that the same model can serve these roles locally under ROCm.
\paragraph{On Comparing to Kevin and Devin-Class Agents.}
Kevin is trained to emit raw CUDA via \texttt{torch.utils.cpp\_extension.load\_inline} on NVIDIA hardware, using CUDA-specific primitives such as \texttt{\_\_shfl\_xor\_sync} and \texttt{\_\_shared\_\_} memory. Our evaluation targets AMD GPUs through HIP/ROCm and TritonBench-G/ROCmBench under ROCm. A same-protocol comparison would require Kevin to emit AMD-compatible Triton or HIP kernels, which it was not trained for. We therefore treat Kevin as evidence that RL on execution feedback is effective in CUDA settings, while focusing our empirical comparison on models evaluated under the same GEAK OptimAgent-v2 protocol. Devin is a hosted coding agent without a benchmark-accessible interface exposing the fixed generator/reflector/optimizer role assignment used here, so we do not include it as a direct baseline.
\begin{table*}[t]
\centering

\small
\setlength{\tabcolsep}{3pt}
\renewcommand{\arraystretch}{1.16}
\begin{tabularx}{\textwidth}{@{}
  >{\raggedright\arraybackslash}p{0.14\textwidth}
  >{\raggedright\arraybackslash}p{0.15\textwidth}
  >{\raggedright\arraybackslash}p{0.16\textwidth}
  >{\raggedright\arraybackslash}p{0.12\textwidth}
  Y
  Y
  @{}}
\toprule
\textbf{System} & \textbf{Target} & \textbf{Training} & \textbf{\makecell[l]{Multi-turn\\feedback}} & \textbf{\makecell[l]{Performance\\signal}} & \textbf{\makecell[l]{Role\\internalization}} \\
\midrule
Kevin & CUDA / KernelBench & Multi-turn RL & Yes & Correctness + speedup & Kernel refinement, not AMD/GEAK roles \\
AutoTriton & Triton & SFT + GRPO/RLVR & No explicit multi-role loop & Rule + execution rewards & No generator/reflector/optimizer loop \\
GEAK & AMD Triton & No model training & Inference-time only & Correctness + latency & Separate frontier-LLM calls \\
\textbf{AMDKernelVault} & \textbf{AMD HIP/Triton} & \textbf{SFT + execution-aware RL} & \textbf{Yes} & \textbf{Compile + correctness + speedup} & \textbf{One 8B policy for all LLM roles} \\
\bottomrule
\end{tabularx}
\begin{flushleft}
\footnotesize \emph{References:} Kevin~\citep{baronio2026kevin}; AutoTriton~\citep{li2025autotriton}; GEAK~\citep{wang2025geak}.
\end{flushleft}
\caption{Positioning relative to closely related kernel-generation systems.}
\label{tab:related_positioning_appendix}
\end{table*}

\section{Dataset Details}
\label{app:datasets}

\subsection{HIP Kernel Datasets}

\subsubsection{HIP-CudaAgent Dataset}
The HIP-CudaAgent Dataset is derived from \textit{CUDA-Agent-Ops-6K} using the HIPKernelGen pipeline:

1. \textbf{PyTorch Task Normalization.} The synthesized operator-level PyTorch tasks are converted into function-style references with input construction and correctness-checking logic.

2. \textbf{HIP Kernel Generation.} An LLM-based generator produces HIP implementations for each normalized PyTorch task. Each candidate is compiled with the ROCm toolchain and validated against the PyTorch reference through execution-based tests.

3. \textbf{Iterative Refinement.} Compilation errors, runtime failures, and numerical mismatches trigger reflector-guided regeneration until correctness is achieved or the maximum attempt limit is reached.

Final: 5,388 validated PyTorch-to-HIP samples.

\subsubsection{HIP-GPUMode Dataset}

The HIP-GPUMode Dataset is generated from the GPUMODE-KernelBook dataset using the HIPKernelGen pipeline:

\begin{enumerate}
\item \textbf{PyTorch Module to Function Conversion.} The first-stage agent converts PyTorch module-level code into function-style PyTorch code with correctness-checking logic to filter functionally equivalent pairs.
\item \textbf{HIP Kernel Generation.} The second-stage agent translates each PyTorch module into HIP code using an LLM-based generator. Each candidate undergoes compilation, execution-based correctness validation, and latency measurements.
\item \textbf{Iterative Refinement.} Failed kernels trigger reflector-guided feedback for iterative regeneration until valid or maximum attempts reached.
\end{enumerate}

Final: \textbf{5,910} unique PyTorch modules with \textbf{22,397} HIP kernel optimization variants (avg 3.8 variants/entry).

\subsubsection{HIP-to-HIP Optimization Pairs}

We extend HIPKernelGen to support kernel-level HIP-to-HIP optimization by operating directly on existing HIP kernel functions annotated with \texttt{\_\_global\_\_} or \texttt{\_\_device\_\_}. The input is a baseline HIP kernel and the output is an optimized kernel with equivalent semantics. Final: \textbf{34,368} validated HIP kernel optimization pairs.

\subsection{Triton Kernel Dataset}

We construct a large-scale Triton kernel dataset using TritonKernelGen, aggregating multiple complementary sources under a unified validation pipeline.

\subsubsection{Stack-v2-dedup (curated)}
We filter and extract PyTorch code containing \texttt{torch} operators and/or Torch Inductor constructs, wrap entries into function-style references, and validate against a high-coverage test suite. From $\sim$13,000 candidates, we obtain \textbf{2,269} medium-difficulty, functionally correct Triton--PyTorch pairs.

\subsubsection{TritonBench-train (modified)}
We curate \textbf{7,713} samples from the original TritonBench-8k. Synthetic kernels generated via the \textit{ninetoothed} DSL are refined or discarded to remove DSL-specific artifacts and input-specific patterns. We expand test coverage to ensure generalization across shapes, strides, and precision types.

\subsubsection{GPUMODE-KernelBook}
We incorporate $\sim$\textbf{18,000} PyTorch entries with Triton code extracted via Torch Inductor, using them as a primary source for operator coverage and structural guidance.

\subsubsection{AI-CUDA-Engineer-Archive}
We leverage \textbf{11,911} usable PyTorch + CUDA pairs. Where Triton code is absent, the agent synthesizes Triton implementations against the shared PyTorch reference and expanded test suite.

\paragraph{Summary.}
Across sources, TritonKernelGen aggregates \textbf{39,893} validated samples. Available metadata varies by source subset and packaged representation; where present, it records construction-time validation outcomes, latency, and precision/shape coverage.

\subsection{ROCm Libraries QA Dataset}

Interface-level extraction from AMD ROCm Libraries using depth-first call graph traversal:

\begin{itemize}[leftmargin=*,itemsep=2pt]
\item \textbf{rocSOLVER}: 519 entries (226 kernel implementations, 293 QA explanations). Difficulty: 190 L1, 198 L2, 131 L3.
\item \textbf{rocBLAS}: 1,858 entries (1,169 implementations, 689 QA explanations). Difficulty: 971 L1, 506 L2, 381 L3.
\end{itemize}

Total: \textbf{2,377} production-grounded supervision entries capturing production HIP development patterns, interface-driven kernel logic, and host--device orchestration.

\subsection{Difficulty Distribution}

Table~\ref{tab:difficulty_distribution} shows the full difficulty statistics used to support the curriculum discussion in Section~\ref{sec:datasets}. Level 1 contains standalone kernels with relatively direct data flow, Level 2 covers fused or single-file implementations requiring coordinated memory access, and Level 3 includes multi-file or interface-level code with cross-module dependencies. The Triton corpus is dominated by Level 2 examples, while the ROCm Libraries QA subset contains a larger fraction of Level 3 entries because production interfaces often involve host-side orchestration and multi-file kernel dispatch.

\begin{table*}[h]
\centering
\small
\begin{tabular}{lrrr}
\toprule
\textbf{Level} & \textbf{Triton} & \textbf{rocBLAS} & \textbf{rocSOLVER} \\
\midrule
Level 1 (Standalone) & 11,824 (29.6\%) & 971 & 190 \\
Level 2 (Fused/single-file) & 24,765 (62.1\%) & 506 & 198 \\
Level 3 (Multi-file) & 3,304 (8.3\%) & 381 & 131 \\
\bottomrule
\end{tabular}
\caption{Difficulty distribution across datasets.}
\label{tab:difficulty_distribution}
\end{table*}

\section{Dataset Analysis}
\label{app:analysis}

This section analyzes the resulting datasets in terms of scale, operator coverage, diversity, and architecture-relevant properties.

\subsection{Dataset Summary}

Table~\ref{tab:datasets} summarizes the scale and composition of AMDKernelVault. The corpus contains 64K+ HIP/ROCm supervision entries and approximately 40K Triton kernels, spanning PyTorch-to-kernel generation, kernel-to-kernel optimization, and production-grounded code understanding. The HIP/ROCm portion combines HIP-CudaAgent and HIP-GPUMode for PyTorch-to-HIP generation, HIP2HIP for optimizing existing HIP kernels, and ROCm Libraries QA for interface-grounded supervision from rocBLAS and rocSOLVER.

The Triton portion aggregates four complementary sources into the Triton-Curated corpus: curated web/code examples from Triton-Stack, benchmark-derived kernels from Triton-Bench, Torch-Inductor-derived programs from Triton-GPUMode, and AI-CUDA-Engineer-derived tasks from Triton-AICE. Together, these subsets provide broad AMD-validated coverage across generated operator workloads, benchmark-style kernels, and production-grounded HIP/Triton programming patterns.

\subsection{Operator and Workload Coverage}

Across both HIP and Triton datasets, kernels span common deep learning operator families, including elementwise transforms, reductions, normalization layers (e.g., layer normalization), matrix and tensor transformations, and composite fusion patterns produced by Torch Inductor.

Because workloads are derived from real PyTorch programs and benchmark suites, kernels reflect practical memory access patterns, shape variability, and fusion structures encountered in training and inference workloads, rather than synthetic microkernels.

The ROCm Libraries QA dataset extends coverage to production-grade linear algebra primitives from rocBLAS (e.g., GEMM, TRSM, SYMM) and rocSOLVER (e.g., QR factorization, eigensolvers), providing supervision grounded in multi-file, interface-driven kernel implementations.

\subsection{Kernel Diversity and Optimization Variants}

The datasets exhibit substantial diversity in valid kernel implementations for the same functional operators. In HIP-GPUMode, multiple HIP kernels are provided per PyTorch entry, differing in block sizes, thread mapping strategies, tiling schemes, memory access orderings, and reduction structures.

Similarly, Triton-Curated contains multiple kernel realizations for overlapping operators, capturing variation in program structure, scheduling decisions, and memory layouts at the DSL level. This diversity exposes non-trivial performance trade-offs and supports studies in kernel ranking, selection, and agent-driven optimization.

Such multiplicity of correct implementations is essential for learning-based systems that aim to reason over optimization spaces rather than merely reproducing reference kernels.

\subsection{Data Quality and Bias Mitigation}

Because a large fraction of the corpus is produced by automated agents, we explicitly filter for execution quality rather than trusting raw model outputs. Candidates must compile, pass reference-based tests, and satisfy ROCm execution constraints before entering the released set. For performance-oriented subsets, multiple candidates are generated and the retained variants are selected after latency measurement on AMD hardware. Public expert-level HIP kernels are scarce, so the corpus combines production kernels from ROCm Libraries with rigorously filtered generated variants. This design removes obviously poor kernels while preserving multiple correct implementations when they expose different optimization strategies.

We also mitigate NVIDIA-centric bias in four ways. First, CUDA-derived samples are treated as seeds, not final AMD supervision: they are translated, compiled, and filtered under ROCm. Second, generated HIP variants are re-optimized and measured on AMD hardware rather than accepted solely for syntactic portability. Third, ROCm Libraries provide expert-written HIP examples from \texttt{rocBLAS} and \texttt{rocSOLVER}, anchoring the dataset in production AMD code. Fourth, Triton samples are refactored and tested for shape generality, stride awareness, and AMD-compatible scheduling assumptions. These checks reduce reliance on NVIDIA-specific warp-size, tensor-core, and tiling assumptions while still preserving useful algorithmic structure from upstream corpora.

\subsection{Portability and Correctness Outcomes}

All released HIP kernels compile and pass execution-based correctness checks on AMD ROCm after filtering. Triton-Curated kernels likewise compile and pass correctness validation on AMD GPUs under expanded test suites covering diverse shapes, strides, and precision modes.

Because upstream corpora differ substantially in portability, Portability filtering removes kernels that rely on non-general indexing, device-specific assumptions, or input-specialized schedules. As a result, the released datasets represent AMD-valid kernel implementations rather than raw collections from heterogeneous sources.

\subsection{Validation Metadata}

Validation metadata is subset-dependent. Where available, packaged records include source identifiers, task or reference information, dtype/shape coverage, compiler or runtime status, latency, and speedup relative to the relevant baseline. Some packaged \texttt{verification} fields are empty even when construction-time validation was performed, and rejected-candidate logs were not retained uniformly across pipelines. For ROCm Libraries QA, available provenance can include source paths and interface-level context. We therefore do not claim that every released record contains all fields needed to reconstruct the complete validation trajectory.

\subsection{Architecture-Relevant Properties (Triton)}

To assess whether Triton-Curated reduces common vendor-shaped artifacts present in upstream corpora, we analyze kernel parameterization and code-shape properties that often encode device-specific assumptions, such as distributions of \texttt{num\_warps}, \texttt{num\_stages}, and block-size choices.

Compared to raw crawled or DSL-generated Triton code, Triton-Curated exhibits greater diversity in scheduling parameters and fewer hard-coded constants tied to specific hardware configurations. This reflects both execution-based filtering on AMD and agent-driven refactoring toward stride-aware and shape-general indexing, which are required for correctness across ROCm backends. During generation and validation, prompts and feedback expose AMD-specific constraints such as wavefront behavior, LDS capacity, compute-unit occupancy, and target architectures including gfx942 and gfx950.

\subsection{Comparison with Existing GPU Kernel Datasets}

Existing kernel datasets and benchmarks primarily target CUDA or Nvidia-backed Triton execution, limiting their applicability to ROCm-based platforms. In contrast, HIP-CudaAgent and HIP-GPUMode provide native HIP kernels that execute on AMD GPUs, while Triton-Curated supplies an AMD-validated Triton corpus aligned with the same PyTorch references.

This enables systematic research on AMD-focused kernel optimization, cross-platform transpilation, and heterogeneous GPU performance modeling that is difficult to conduct using Nvidia-centric benchmarks alone.

\subsection{Implications for Learning-Based Kernel Research}

By combining real-world workload coverage, verified executability on AMD GPUs, and diverse optimization variants across both HIP and Triton representations, the released datasets support supervised fine-tuning, reinforcement learning, and agent-based kernel synthesis under realistic hardware constraints.

These properties enable research on performance-aware kernel generation, automatic optimization strategies, and cross-abstraction learning, substantially lowering the barrier to data-driven GPU systems research on AMD platforms.

\section{More Details of Compact LLM Training}
\label{app:training_method_details}

This appendix provides implementation details for the training recipe summarized in Section~\ref{sec:training}.

\subsection{Task Serialization and SFT Objective}

Each training instance is serialized as an instruction--response pair. The instruction contains the task specification, such as PyTorch reference code, a natural-language Triton prompt, or an input HIP/Triton kernel. The response contains the target kernel implementation, optionally preceded by reasoning annotations for Triton examples.

SFT optimizes the standard next-token likelihood:
\begin{equation}
    \mathcal{L}_{\text{SFT}}(\theta)
    =
    -\mathbb{E}_{(x,y^*)}
    \left[
    \log \pi_\theta(y^* \mid x)
    \right],
    \label{eq:sft}
\end{equation}
where $x$ is the task specification and $y^*$ is the target implementation.

For HIP, training examples use direct instruction--response formatting without additional reasoning augmentation. For Triton, when available, we include Chain-of-Thought-style planning before the final kernel:
\[
y^* = [\text{reasoning}, \text{kernel}].
\]
This helps the model learn tiling choices, masking logic, memory-access patterns, and boundary handling. All SFT runs use full-parameter fine-tuning rather than PEFT/LoRA; in preliminary experiments, parameter-efficient adapters underperformed for this setting because the model has to learn both low-level kernel syntax and agentic self-reflection behavior, and the coupled generation and debugging skills required broader updates than adapter-only training provided.

\subsection{Reward Design}

Execution-aware RL uses a hierarchical reward:
\begin{equation}
r = r^{\text{compile}} + r^{\text{correct}} + r^{\text{perf}},
\label{eq:reward}
\end{equation}
where
\[
r^{\text{compile}} = 0.4,\quad
r^{\text{correct}} = 1.0,\quad
r^{\text{perf}} \leq 1.5.
\]
The compilation reward is assigned when the candidate builds successfully under the target backend. The correctness reward is assigned when the candidate passes numerical validation against the PyTorch reference. The performance reward is defined as
\begin{equation}
r^{\text{perf}} =
\min\left(0.375 \cdot (\log_2 s)^2, 1.5\right),
\label{eq:perf_reward}
\end{equation}
for speedup $s \geq 1$.

This design gives partial credit to near-valid kernels that compile but fail correctness, which is important in the low-success regime of kernel RL. The bounded performance term rewards meaningful speedups while reducing sensitivity to noisy latency measurements, JIT effects, or autotuning outliers. Typical fast-converging speedups during RL sit well below the 1.5 cap, so this cap functions as a conservative stability guard rather than a tuned hyperparameter; the reward-sensitivity analysis (Appendix~\ref{app:ablation}, Table~\ref{tab:reward_sensitivity}) confirms low sensitivity to the exact cap value.

\subsection{GRPO Optimization}

We use Group Relative Policy Optimization (GRPO) for execution-aware RL~\citep{shao2024deepseekmath}. Standard GRPO normalizes advantages by reward variance:
\[
\hat{A}_i = \frac{r_i - \bar{r}}{\sigma_r + \epsilon}.
\]
In our setting, many early groups contain uniformly failed samples or samples with very similar partial rewards. We therefore disable variance normalization and use:
\begin{equation}
\hat{A}_i = r_i - \bar{r}.
\label{eq:advantage}
\end{equation}

The policy update follows PPO-style clipping:
\begin{equation}
\mathcal{L}(\theta)
=
-\mathbb{E}
\left[
\min\left(
\rho_t \hat{A}_t,
\mathrm{clip}(\rho_t, 1\pm\epsilon)\hat{A}_t
\right)
\right]
+
\beta D_{\mathrm{KL}}.
\label{eq:ppo}
\end{equation}

\subsection{Multi-Turn On-Policy Reflection}

For Triton, each rollout follows a multi-turn generate--evaluate--reflect process. At turn $t$, the policy observes the task prompt, previous candidate code, compiler output, correctness failures, and optional latency feedback. It then emits both a reflection and a revised kernel. The multi-turn return is:
\begin{equation}
G(\tau) = \sum_{t=1}^{T} \gamma^{t-1} r_t,
\label{eq:multiturn_return}
\end{equation}
with $\gamma = 0.6$ and $T = 3$.

Because the same policy produces both reflection and code, the model learns feedback interpretations that are actionable for its own generation behavior. This aligns training with deployment in GEAK-style agent loops, where the trained model serves as generator, reflector, and optimizer.

\subsection{Curriculum and Hyperparameters}

HIP training uses SFT followed by single-turn GRPO. Triton training uses SFT followed by multi-turn GRPO with on-policy reflection and curriculum scheduling. The Triton curriculum gradually shifts from easier kernels to harder samples, reducing the fraction of Level-1 examples while increasing Level-2 and Level-3 tasks.

We train Qwen3-8B-Instruct with the following configuration:
\begin{itemize}[leftmargin=*,itemsep=1pt]
\item HIP SFT: 8$\times$MI250, 4 epochs.
\item HIP RL: 8$\times$MI325X, 2 epochs, batch size 12, 6 rollouts.
\item Triton SFT: 32$\times$MI325X, 2 epochs.
\item Triton RL: multi-turn GRPO, 2.7 epochs, 3 turns.
\end{itemize}

\subsection{Reward Sensitivity and Design Rationale}

Binary correctness rewards are insufficient in early AMD kernel RL because most candidates fail to compile or fail numerical validation. Under such sparse feedback, many GRPO groups produce zero or near-zero useful advantage. The compile reward increases the fraction of samples with usable signal, while disabled variance normalization prevents small reward differences from being amplified or collapsed by unstable group statistics.

Ablations in Appendix~\ref{app:ablation} and Table~\ref{tab:reward_sensitivity} evaluate these choices. Removing multi-turn reflection, removing SFT cold start, or weakening the hierarchical reward reduces performance, confirming that execution-aware training is important for turning AMDKernelVault into an effective training resource for compact AMD kernel LLMs.

\section{Evaluation Details}
\label{app:eval_details}

\paragraph{Benchmark scope.}
Our execution-intensive evaluation uses 200 PyTorch-to-HIP tasks, 100 HIP-to-HIP optimization tasks, 184 TritonBench-G kernels, and 31 ROCmBench kernels. These sets require compilation and on-device validation and are suitable for controlled comparisons, but they do not establish performance over every AMD operator or workload family.

\paragraph{Agent-budget sensitivity.}
Our main results use a fixed 3-iteration budget. On TritonBench-G, Q3-8B+SFT+RL improves from 12.0\% Corr@1 to 33.2\% Corr@3, compared with 7.6\% to 29.8\% for Claude Sonnet 4 and 3.8\% to 15.2\% for GPT-5. On ROCmBench, it improves from 22.58\% to 41.94\%, compared with Claude Sonnet 4 from 19.35\% to 35.48\%.

\section{Training Details}
\label{app:training}

\subsection{Hyperparameters}

\begin{table}[h]
\centering

\small
\begin{tabular}{lcc}
\toprule
\textbf{Parameter} & \textbf{HIP} & \textbf{Triton} \\
\midrule
\multicolumn{3}{l}{\textit{Model}} \\
Base model & \multicolumn{2}{c}{Qwen3-8B-Instruct} \\
Sequence length & 8K & 16K \\
\midrule
\multicolumn{3}{l}{\textit{SFT}} \\
Hardware & 8$\times$MI250 & 32$\times$MI325X \\
Learning rate & $2 \times 10^{-5}$ & $1 \times 10^{-5}$ \\
Batch size & 8/GPU & 8/GPU \\
Epochs & 4 & 2 \\
Framework & LLaMA-Factory & LLaMA-Factory \\
\midrule
\multicolumn{3}{l}{\textit{RL: GRPO / Multi-Turn GRPO}} \\
Hardware & 8$\times$MI325X & 32$\times$MI325X \\
Learning rate & $1 \times 10^{-6}$ & $1 \times 10^{-6}$ \\
Batch size & 12 & 64 prompts \\
Group size $G$ & 6 & 8 \\
Clip $\epsilon$ & 0.2 & 0.2 / 0.28 \\
KL coefficient $\beta$ & 0.01 & 0.01 \\
Framework & VeRL & SLIME \\
Epochs & 2 & 2.7 \\
\midrule
\multicolumn{3}{l}{\textit{Multi-Turn (Triton only)}} \\
Turns $T$ & -- & 3 \\
Discount $\gamma$ & -- & 0.6 \\
\midrule
\multicolumn{3}{l}{\textit{Reward}} \\
Compile partial & -- & 0.4 \\
Correct & -- & 1.0 \\
Performance cap & -- & 1.5 \\
\bottomrule
\end{tabular}
\caption{Complete hyperparameter configuration.}
\end{table}

\subsection{Implementation Components}

Our implementation separates training, execution feedback, and benchmark scoring into modular components. The multi-turn RL trainer follows a SLIME-style rollout loop: prompts are sampled from the kernel task dataset, the policy generates candidate code and reflections for up to $T=3$ turns, and GRPO updates the same policy using the hierarchical rewards described in Section~\ref{sec:training}. Reward-standard-deviation normalization is disabled to preserve stable advantages in the sparse-reward regime. A separate execution server compiles candidates, runs correctness tests, and measures latency, returning compiler/runtime traces and performance feedback to the rollout loop. Final benchmark scoring uses TritonBench-G and ROCmBench, while the headline Triton results use the fixed agent-training evaluator (GEAK OptimAgent-v2) with \texttt{num\_offsprings=1}, \texttt{max\_iteration=3}, and no profiler feedback.

\subsection{Curriculum Learning Phases (Triton)}

For Triton RL training, we employ difficulty fading across three phases:

\begin{table}[h]
\centering
\small

\begin{tabular}{clr}
\toprule
\textbf{Phase} & \textbf{Composition} & \textbf{Samples} \\
\midrule
1 & L1 only (Easy) & 5.3K \\
2 & L1$_{\text{subset}}$ + L2 + L3$_{\text{med}}$ & 16.5K \\
3 & L2$_{\text{subset}}$ + L3 + variants & 4K \\
\bottomrule
\end{tabular}
\caption{Curriculum phases for Triton RL training.}
\label{tab:curriculum}
\end{table}

Unlike cumulative curricula that add harder samples while keeping all easy samples, difficulty fading progressively reduces easy samples to focus training on harder cases. This approach prevents the model from overfitting to easy examples while ensuring sufficient exposure to challenging multi-file and fused kernel implementations.

\section{Training Dynamics}
\label{app:dynamics}

Figure~\ref{fig:training} presents the training dynamics for our multi-turn on-policy reflection RL on Triton kernel generation. The reward trajectory (a) shows steady improvement from the SFT baseline ($\sim$0.19) to peak values above 1.0, demonstrating effective learning from execution feedback. The benchmark evaluation (b) tracks performance on TritonBench-G (compilation and correctness accuracy) and ROCmBench throughout training, showing consistent gains across both benchmarks.

\begin{figure*}[h!]
\centering
\includegraphics[width=0.8\linewidth]{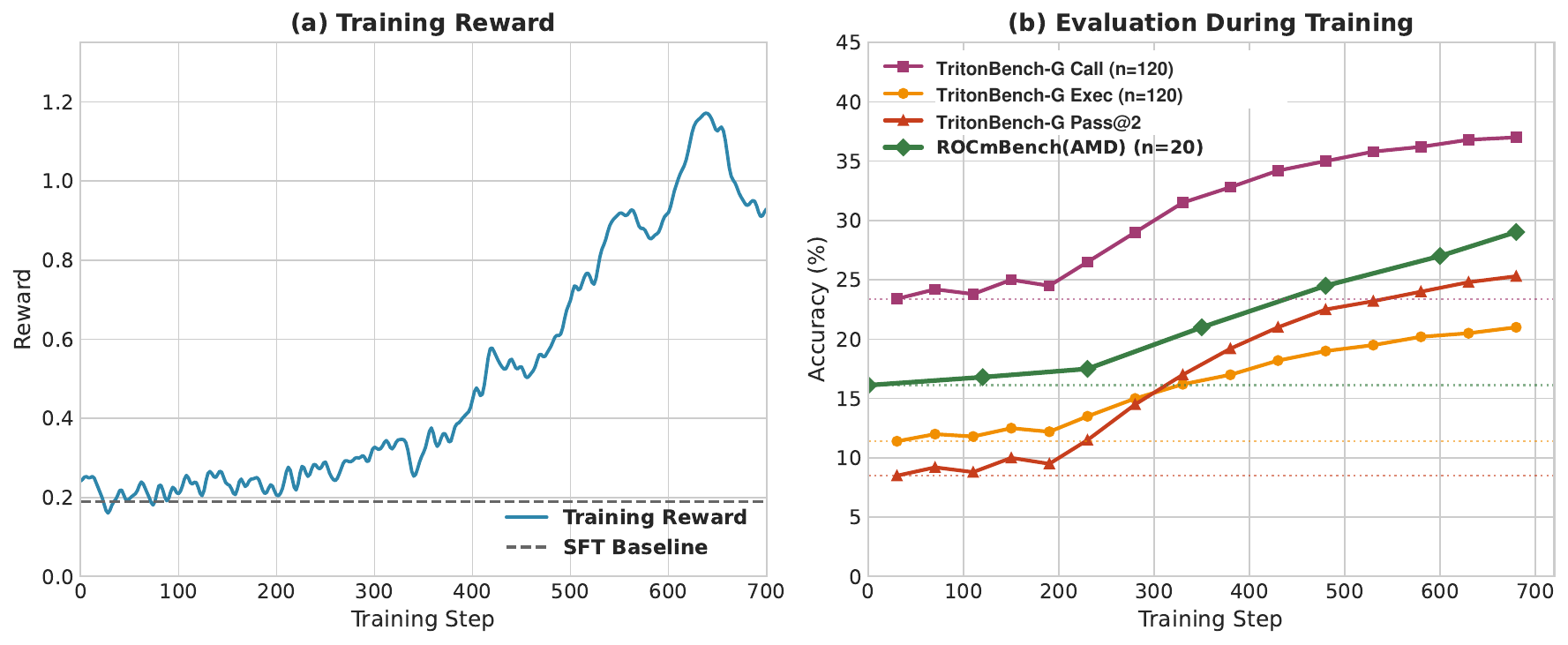}
\caption{RL training results for Triton kernel generation. (a) Reward trajectory showing improvement from SFT baseline ($\sim$0.19) to peak values above 1.0. (b) Benchmark evaluation during training.}
\label{fig:training}
\end{figure*}

\section{Additional Experimental Analysis}
\label{app:additional_experiments}

\paragraph{Cost and deployability.}
A practical advantage of training one compact policy is that the GEAK-style loop can run locally on AMD GPUs. Frontier-LLM-based pipelines require repeated external calls for generation, reflection, and optimization, introducing API cost, latency variance, and deployment constraints. In contrast, our trained Qwen3-8B model runs within the local AMD environment and serves all three roles in the agent loop.

\paragraph{Cross-hardware transfer.}
To test whether learned optimizations overfit to one AMD target, we evaluate HIP kernels optimized on MI325 GPUs on MI250 GPUs. The optimized kernels retain substantial gains, reaching Speed@10 of 1.68$\times$ on MI250. Triton evaluations on MI325 and MI350 show similar qualitative trends, suggesting that the model learns transferable optimization principles such as locality, tiling, and parallel reduction patterns while still benefiting from ROCm-specific validation.

\paragraph{Additional failure modes.}
Beyond D5 expert-level failures, we observe three common categories of errors: shape-specialized indexing that fails on broader test cases, numerically unstable reductions, and inefficient memory layouts inherited from CUDA/NVIDIA-centric assumptions. These findings motivate future extensions with profiler-guided training feedback and additional production AMD kernels.

\section{Ablation Studies}
\label{app:ablation}

We conduct comprehensive ablation studies to validate our design choices across both HIP and Triton training pipelines.

\subsection{Data Scale (HIP)}

Table~\ref{tab:ab_data_scale} evaluates the impact of training data scale on PyTorch-to-HIP translation (200 tasks from KernelBench Levels 1 and 2).

\begin{table}[h]
\centering

\small
\begin{tabular}{c|cc}
\toprule
\textbf{Data Rate} & \textbf{Comp Acc} & \textbf{Corr Acc} \\
\midrule
0\% (base) & 50.0\% & 14.5\% \\
10\% & 62.5\% & 17.5\% \\
30\% & 64.5\% & 18.0\% \\
50\% & 69.5\% & 23.0\% \\
\rowcolor{green!10} 100\% & \textbf{83.5\%} & \textbf{34.0\%} \\
\bottomrule
\end{tabular}
\caption{Impact of training data scale on PyTorch-to-HIP translation.}
\label{tab:ab_data_scale}
\end{table}

With only 10\% of training data, the model achieves 62.5\% compilation and 17.5\% correctness—already surpassing the base model (50.0\%/14.5\%). Performance scales consistently, confirming our dataset provides high-quality, diverse supervision.

\subsection{Triton RL Components}

Table~\ref{tab:ablation} presents ablations on the Triton training framework, evaluated using the agent-training evaluator (GEAK OptimAgent-v2) with 3 iterations on TritonBench-G (184 kernels).

\begin{table}[h]
\centering

\small
\begin{tabular}{lcc}
\toprule
\textbf{Configuration} & \textbf{Corr Acc} & \textbf{$\Delta$} \\
\midrule
Full system & 33.2\% & -- \\
\midrule
Single-turn ($T=1$) & 20.1\% & -13.1 \\
No SFT cold start & 8.5\% & -24.7 \\
\bottomrule
\end{tabular}
\caption{Ablation studies on Triton generation.}
\label{tab:ablation}
\end{table}

Key findings:
\begin{itemize}[leftmargin=*,itemsep=2pt]
\item \textbf{Multi-turn is essential} (+13.1\%): Single-turn generation cannot recover from initial failures, while 3-turn reflection enables progressive refinement. The model learns to identify errors from execution feedback and generate targeted fixes.
\item \textbf{SFT cold start is critical} (+24.7\%): Without SFT initialization, the base model's low initial success rate ($\sim$15\%) produces extremely sparse RL gradients. SFT raises this to $\sim$51\%, enabling effective RL optimization.
\end{itemize}

\subsection{Fixed Agent Scaffold Ablation}

Table~\ref{tab:fixed_agent_scaffold} isolates the effect of the model training stage under a fixed agent-training evaluator. All rows use the same GEAK OptimAgent-v2 loop on TritonBench-G with \texttt{num\_offsprings=1}, \texttt{max\_iteration=3}, no profiler feedback, and identical generator/reflector/optimizer role assignment; only the model checkpoint changes. The large gain from base Q3-8B to Q3-8B+SFT+Agentic RL shows that the external agent scaffold alone is insufficient: the policy must be trained on agent feedback trajectories to use compiler, correctness, and performance feedback effectively.

\begin{table}[h]
\centering

\small
\setlength{\tabcolsep}{4pt}
\renewcommand{\arraystretch}{1.12}
\begin{tabularx}{\columnwidth}{@{}Y Y c@{}}
\toprule
\textbf{Model inside agent evaluator} & \textbf{Training signal} & \textbf{Corr@3} \\
\midrule
Q3-8B (base) & None / role-agnostic base model & 4.9\% \\
Q3-8B + SFT & Supervised kernel data & 10.3\% \\
\rowcolor{green!10} Q3-8B + SFT + Agentic RL & Multi-turn compiler/test/perf feedback & \textbf{33.2\%} \\
\bottomrule
\end{tabularx}
\caption{Effect of agentic training under a fixed agent scaffold. All rows use the same GEAK OptimAgent-v2 loop; only the model checkpoint changes.}
\label{tab:fixed_agent_scaffold}
\end{table}

\subsection{Reward Sensitivity}

We further analyze the components of the hierarchical reward used during Triton RL. The dominant sensitivity is not the absolute reward scale but whether the reward creates usable gradient signal in the low-success regime.

\begin{table*}[t]
\centering

\small
\setlength{\tabcolsep}{5pt}
\renewcommand{\arraystretch}{1.12}
\begin{tabularx}{\textwidth}{@{}l l Y@{}}
\toprule
\textbf{Factor} & \textbf{Sensitivity} & \textbf{Observed effect} \\
\midrule
Compile partial credit & High & Raises usable gradient mass from near-zero to $\sim$44\% \\
Disable $\sigma$-normalization & Critical & Avoids advantage collapse (Adv. $\approx 0.5$ vs. $\sim10^{-8}$) \\
Credit value 0.2 vs. 0.4 & Low--medium & Both learn; 0.4 gives denser early compile signal \\
Performance cap & Low & Stabilizes late training against noisy $>$15$\times$ speedups \\
\bottomrule
\end{tabularx}
\caption{Qualitative sensitivity of reward design choices in sparse Triton RL.}
\label{tab:reward_sensitivity}
\end{table*}

These observations motivate the chosen reward in Eq.~\ref{eq:reward}: partial compilation credit provides early learning signal, disabled normalization prevents variance collapse or explosion, and the bounded log-squared performance term rewards real speedups without letting measurement jitter dominate correctness.

\section{Extended Results}
\label{app:results}

\subsection{Pass@K Analysis (Parallel Scaling)}

Table~\ref{tab:passk} presents pass@k results on ROCmBench (31 Triton kernels), measuring the probability of generating at least one correct kernel within $k$ samples. Our model achieves the highest pass@k across all values, with pass@5 reaching 41.94\% compared to 38.71\% for Claude Sonnet 4.

\begin{table*}[h]
\centering

\small
\begin{tabular}{lccccc}
\toprule
\textbf{Model} & \textbf{pass@1} & \textbf{pass@2} & \textbf{pass@3} & \textbf{pass@4} & \textbf{pass@5} \\
\midrule
\rowcolor{green!10} \textbf{Ours} & \textbf{27.10\%} & \textbf{35.16\%} & \textbf{38.71\%} & \textbf{40.65\%} & \textbf{41.94\%} \\
Claude Sonnet 4 & 25.81\% & 33.87\% & 36.77\% & 38.06\% & 38.71\% \\
GPT-5 & 20.65\% & 28.39\% & 31.29\% & 32.26\% & 32.26\% \\
Gemini 2.5 Pro & 17.42\% & 24.84\% & 29.03\% & 32.26\% & 35.48\% \\
\bottomrule
\end{tabular}
\caption{Pass@K results on ROCmBench (31 kernels, 5 samples each with temp=1.0)}
\label{tab:passk}
\end{table*}

\begin{figure*}[h]
\centering
\includegraphics[width=0.7\linewidth]{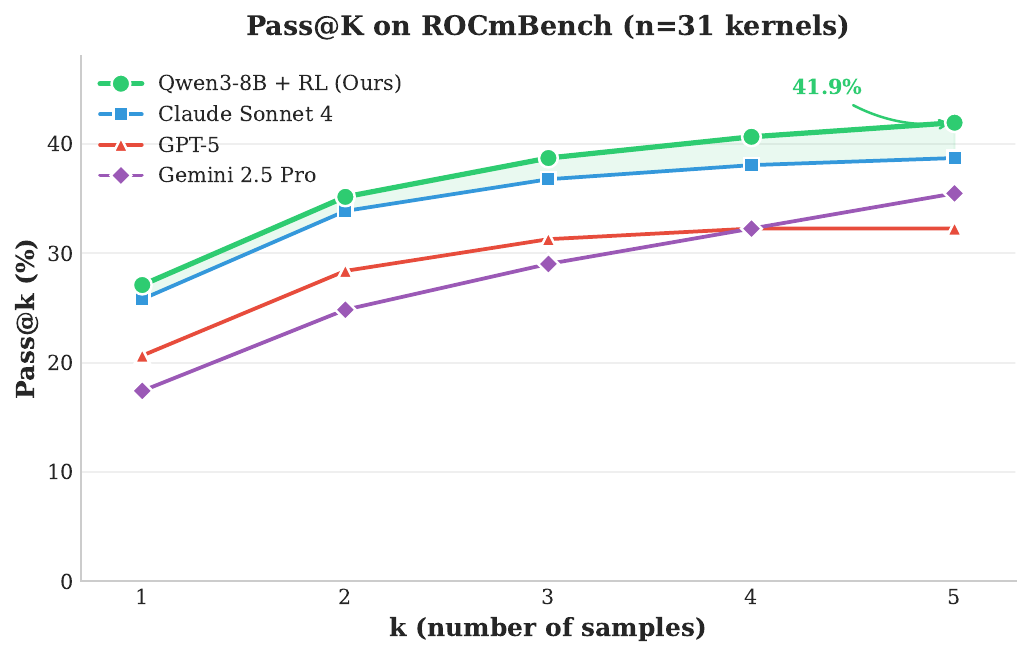}
\caption{Pass@K comparison on ROCmBench (31 kernels). Our model achieves the highest pass rate at all $k$ values, demonstrating superior sample efficiency for kernel generation.}
\label{fig:passk}
\end{figure*}

\subsection{Full Difficulty Breakdown}

Table~\ref{tab:difficulty_full} provides the complete difficulty-wise breakdown including all model variants, showing the progression from base model through SFT to RL.

\begin{table*}[h]
\centering

\small
\begin{tabular}{lcccccc}
\toprule
\textbf{Difficulty} & \textbf{Claude} & \textbf{GPT-5} & \textbf{Gemini} & \textbf{Base} & \textbf{SFT} & \textbf{Ours} \\
\midrule
D1 (Easy, 3)      & 1/3 & 0/3 & 2/3 & 2/3 & 3/3 & \textbf{3/3} \\
D2 (Med, 27)     & 12/27 & 7/27 & 9/27 & 7/27 & 10/27 & \textbf{17/27} \\
D3 (Hard, 65)    & 24/65 & 16/65 & 21/65 & 0/65 & 6/65 & \textbf{28/65} \\
D4 (V.Hard, 84)  & \textbf{18/84} & 5/84 & 11/84 & 0/84 & 0/84 & 13/84 \\
D5 (Expert, 5)   & 0/5 & 0/5 & 0/5 & 0/5 & 0/5 & 0/5 \\
\midrule
\textbf{Total}   & 55 & 28 & 43 & 9 & 19 & \textbf{61} \\
\bottomrule
\end{tabular}
\caption{Complete difficulty-wise results (exec-correct kernels at final iteration).
Base = Qwen3-8B, SFT = Cold-Start SFT, Ours = SFT $\rightarrow$ Multi-Turn RL.}
\label{tab:difficulty_full}
\end{table*}

\subsection{Training Progression}

The progression from Base (4.9\%) to SFT (10.3\%) to RL (33.2\%) demonstrates the effectiveness of each training stage:
\begin{itemize}[leftmargin=*,itemsep=0pt]
\item \textbf{Base $\rightarrow$ SFT}: +5.4\% (2.1$\times$) — learns kernel syntax and patterns
\item \textbf{SFT $\rightarrow$ RL}: +22.9\% (3.2$\times$) — learns from execution feedback
\end{itemize}

\section{Theoretical Analysis}
\label{app:proofs}

\subsection{Zero-Advantage Collapse in GRPO}

Standard GRPO uses binary rewards (correct/incorrect) with group-normalized advantages. We first prove why this leads to frequent zero-gradient updates, then explain how our design avoids this failure mode.

\subsubsection{Binary Reward Collapse (Motivation)}

\begin{theorem}[Binary Reward Collapse]
\label{thm:collapse}
In standard GRPO with group size $G$ and binary rewards $r \in \{0, 1\}$, let $p_e$ denote the probability that a single sample succeeds (receives reward 1). The probability that all samples in a group receive identical rewards, yielding zero advantages, is:
\begin{equation}
P(\hat{A}_i = 0, \forall i) = p_e^G + (1-p_e)^G
\end{equation}
\end{theorem}

\begin{proof}
In standard GRPO, the advantage for sample $i$ is:
\begin{equation}
\hat{A}_i = \frac{r_i - \bar{r}}{\sigma_r + \epsilon}
\end{equation}
where $\bar{r} = \frac{1}{G}\sum_{j=1}^G r_j$ is the group mean and $\sigma_r$ is the group standard deviation.

For binary rewards $r_i \in \{0, 1\}$, the advantage $\hat{A}_i = 0$ for all $i$ if and only if all samples receive the same reward:
\begin{itemize}
\item If all $r_i = 0$: then $\bar{r} = 0$, so $r_i - \bar{r} = 0$ for all $i$
\item If all $r_i = 1$: then $\bar{r} = 1$, so $r_i - \bar{r} = 0$ for all $i$
\end{itemize}

Since samples are drawn independently:
\begin{align}
P(\text{all fail}) &= (1-p_e)^G \\
P(\text{all pass}) &= p_e^G \\
P(\text{zero advantage}) &= p_e^G + (1-p_e)^G
\end{align}
\end{proof}

\begin{corollary}[Collapse Rate with Binary Rewards]
For typical kernel generation with $\sim$6\% initial success rate ($p_e = 0.06$) and group size $G = 8$:
\begin{align}
P(\text{zero advantage}) &= (0.06)^8 + (0.94)^8 \\
&\approx 0 + 0.606 = 60.6\%
\end{align}
Over 60\% of training batches produce zero gradients—the model cannot learn.
\end{corollary}

\subsubsection{Our Solution: Hierarchical Rewards + Disabled Normalization}

To avoid binary-reward collapse, we combine two mechanisms:

\paragraph{1. Hierarchical Rewards.}
Instead of binary $r \in \{0, 1\}$, we use:
\begin{equation}
r = \underbrace{r^{\text{compile}}}_{0.4} + \underbrace{r^{\text{correct}}}_{1.0} + \underbrace{r^{\text{perf}}}_{\leq 1.5} \in [0, 2.9]
\end{equation}

This creates \textbf{intermediate reward levels}: a sample that compiles but fails correctness receives $r = 0.4$, distinguishing it from complete failures ($r = 0$). This makes identical rewards across a group less likely.

\paragraph{2. Disabled $\sigma$-Normalization.}
Even with hierarchical rewards, early training can exhibit reward clustering (e.g., most samples get $r = 0$ or $r = 0.4$). When $\sigma_r$ is small but non-zero, the standard normalization $\hat{A}_i = (r_i - \bar{r})/\sigma_r$ amplifies small differences into large, unstable gradients.

During training, we observed training collapse due to these unstable gradients when using standard normalization. Disabling normalization entirely resolved this:
\begin{equation}
\hat{A}_i = r_i - \bar{r}
\end{equation}

The combination ensures: (1) hierarchical rewards create variance by distinguishing partial successes, and (2) disabled normalization keeps gradients stable when that variance is small.

\paragraph{Curriculum-Induced Clustering.}
During curriculum training, batches contain samples of similar difficulty (e.g., all L1-Easy). Even successful samples receive similar rewards ($\approx$1.4), causing $\sigma_r \approx 0$. Disabled normalization handles this gracefully—any differential progress (one sample achieving performance bonus while others don't) produces non-zero gradients.

\subsection{On-Policy Reflection}

At each turn, the policy conditions on the prompt, prior code and reflections, and the compiler or execution feedback returned by the environment. The same parameters generate subsequent revisions and are updated from the resulting trajectory return. We use this standard on-policy construction as an implementation choice; we do not claim a new policy-gradient result.

\subsection{Multi-Turn Credit Assignment}

\begin{lemma}[Temporal Credit with $\gamma < 1$]
For discount factor $\gamma = 0.6$ and $T = 3$ turns, the effective contribution weights are:
\begin{align}
w_1 &= 1.0 \quad \text{(Turn 1)} \\
w_2 &= 0.6 \quad \text{(Turn 2)} \\
w_3 &= 0.36 \quad \text{(Turn 3)}
\end{align}
After normalization, this assigns approximately 51.0\% of the weight to Turn 1, encouraging the model to succeed early while still rewarding improvement through reflection.
\end{lemma}

\section{Multi-Turn Self-Correction Examples}
\label{app:examples}

We present abridged pseudocode examples illustrating self-correction on \textbf{medium-to-hard difficulty} kernels. Ellipses and shortened fragments omit boilerplate and therefore are not standalone executable listings.

\subsection{Example 1: Cross-Attention (Medium-Hard)}

\textbf{Task:} Implement cross-attention softmax kernel for RGB-Depth fusion.

\textbf{Concepts:} Multi-head attention, softmax fusion, batched matrix operations, numerical stability.

\paragraph{Turn 1} (Compile=\cmark, Exec=\xmark, Reward=0.4): Missing numerical stability (max subtraction):
\begin{lstlisting}
@triton.jit
def softmax_kernel(input_ptr, output_ptr, n_rows, n_cols,
                   BLOCK_SIZE: tl.constexpr):
    row_idx = tl.program_id(0)
    col_offsets = tl.arange(0, BLOCK_SIZE)
    vals = tl.load(input_ptr + row_idx * n_cols + col_offsets)
    # BUG: No numerical stability - causes overflow!
    output = tl.exp(vals) / tl.sum(tl.exp(vals))
    tl.store(output_ptr + row_idx * n_cols + col_offsets, output)
\end{lstlisting}

\textbf{Feedback:} ``NaN values in output. Softmax requires numerical stability: subtract max before exp() to prevent overflow.''

\paragraph{Turn 2} (Compile=\cmark, Exec=\cmark, Reward=1.87, Speedup=2.17$\times$): Correct numerically stable softmax:
\begin{lstlisting}
@triton.jit
def softmax_kernel(input_ptr, output_ptr, n_rows, n_cols,
                   input_row_stride, BLOCK_SIZE: tl.constexpr):
    row_idx = tl.program_id(0)
    row_start = input_ptr + row_idx * input_row_stride
    col_offsets = tl.arange(0, BLOCK_SIZE)
    mask = col_offsets < n_cols
    vals = tl.load(row_start + col_offsets, mask=mask, other=-inf)
    # FIXED: Numerical stability
    row_max = tl.max(vals, axis=0)
    vals_shifted = vals - row_max
    numerator = tl.exp(vals_shifted)
    output = numerator / tl.sum(numerator, axis=0)
    tl.store(output_ptr + row_idx * n_cols + col_offsets, output, mask=mask)
\end{lstlisting}

\textbf{Recovery:} Immediate (1 turn). Final speedup: \textbf{2.17$\times$}.

\subsection{Example 2: PreNorm Transformer Decoder (Hard)}

\textbf{Task:} Implement Triton-optimized Pre-Norm Transformer Decoder Layer.

\textbf{Concepts:} Fused layer normalization, multi-head attention, cross-attention, FFN, residual connections.

\paragraph{Turn 1} (Compile=\xmark, Exec=\xmark, Reward=0.0): Incorrect kernel structure for layer norm:
\begin{lstlisting}
@triton.jit
def layer_norm_kernel(x_ptr, weight_ptr, output_ptr, M, N):
    # BUG: Missing eps, wrong reduction, no bias
    row = tl.program_id(0)
    mean = tl.sum(tl.load(x_ptr + row * N)) / N  # Wrong load!
    var = tl.sum((x - mean) ** 2) / N
    output = (x - mean) / tl.sqrt(var)  # Missing eps!
\end{lstlisting}

\textbf{Feedback:} ``Compilation failed: undefined 'x'. Also: missing eps causes division by zero, missing bias parameter.''

\paragraph{Turn 2} (Compile=\cmark, Exec=\xmark, Reward=0.4): Fixed syntax but wrong reduction for large N:
\begin{lstlisting}
@triton.jit
def layer_norm_kernel(x_ptr, weight_ptr, bias_ptr, output_ptr,
                      M, N, eps: tl.constexpr, BLOCK_SIZE: tl.constexpr):
    row = tl.program_id(0)
    cols = tl.arange(0, BLOCK_SIZE)
    # BUG: Only handles N <= BLOCK_SIZE
    x = tl.load(x_ptr + row * N + cols, mask=cols < N)
    mean = tl.sum(x) / N
    var = tl.sum((x - mean) ** 2) / N
    ...
\end{lstlisting}

\textbf{Feedback:} ``Incorrect output for N > BLOCK\_SIZE. Use multi-pass reduction for large feature dimensions.''

\paragraph{Turn 3} (Compile=\cmark, Exec=\cmark, Reward=2.9, Speedup=9.13$\times$): Correct multi-pass layer norm:
\begin{lstlisting}
@triton.jit
def layer_norm_kernel(x_ptr, weight_ptr, bias_ptr, output_ptr,
                      M, N, eps: tl.constexpr, BLOCK_SIZE: tl.constexpr):
    row = tl.program_id(0)
    # FIXED: Multi-pass for arbitrary N
    mean = 0.0
    for off in range(0, N, BLOCK_SIZE):
        cols = off + tl.arange(0, BLOCK_SIZE)
        mask = cols < N
        x = tl.load(x_ptr + row * N + cols, mask=mask, other=0.0)
        mean += tl.sum(x, axis=0)
    mean = mean / N
    # Similar multi-pass for variance...
\end{lstlisting}

\textbf{Recovery:} Progressive (2 turns). Final speedup: \textbf{9.13$\times$}.

\subsection{Example 3: Bound Softmax Multi-Pass (Medium)}

\textbf{Task:} Implement numerically stable softmax for sequences longer than block size.

\textbf{Concepts:} 3-pass algorithm (max, exp+sum, normalize), arbitrary sequence length handling.

\paragraph{Turn 1} (Compile=\cmark, Exec=\xmark, Reward=0.4): Single-pass fails for long sequences:
\begin{lstlisting}
@triton.jit
def softmax_kernel(input_ptr, output_ptr, n_rows, n_cols,
                   BLOCK_SIZE: tl.constexpr):
    row_idx = tl.program_id(0)
    # BUG: Assumes n_cols <= BLOCK_SIZE
    cols = tl.arange(0, BLOCK_SIZE)
    vals = tl.load(input_ptr + row_idx * n_cols + cols)
    max_val = tl.max(vals)  # Only partial max!
    exp_vals = tl.exp(vals - max_val)
    output = exp_vals / tl.sum(exp_vals)
\end{lstlisting}

\textbf{Feedback:} ``Incorrect softmax for n\_cols=2048 > BLOCK\_SIZE=1024. Partial max causes incorrect normalization.''

\paragraph{Turn 2} (Compile=\cmark, Exec=\cmark, Reward=2.9, Speedup=5.18$\times$): Correct 3-pass algorithm:
\begin{lstlisting}
@triton.jit
def bound_softmax_kernel(input_ptr, output_ptr, n_rows, n_cols,
                         BLOCK_SIZE: tl.constexpr):
    row_idx = tl.program_id(0)
    row_start = input_ptr + row_idx * n_cols
    # PASS 1: Find global maximum
    max_val = float('-inf')
    for col_off in range(0, n_cols, BLOCK_SIZE):
        cols = col_off + tl.arange(0, BLOCK_SIZE)
        mask = cols < n_cols
        vals = tl.load(row_start + cols, mask=mask, other=float('-inf'))
        max_val = tl.maximum(max_val, tl.max(vals, axis=0))
    # PASS 2: Compute exp(x-max) and sum
    sum_val = 0.0
    for col_off in range(0, n_cols, BLOCK_SIZE):
        ...  # Store intermediate, accumulate sum
    # PASS 3: Normalize by sum
    for col_off in range(0, n_cols, BLOCK_SIZE):
        ...  # Divide by sum_val
\end{lstlisting}

\textbf{Recovery:} Immediate (1 turn). Final speedup: \textbf{5.18$\times$}.

\subsection{Summary: Recovery Statistics}

\begin{table}[t]
\centering
\small
\setlength{\tabcolsep}{4pt}
\renewcommand{\arraystretch}{1.08}
\begin{tabular}{llcccc}
\toprule
\textbf{Kernel} & \textbf{Diff.} & \textbf{T1} & \textbf{T2} & \textbf{T3} & \textbf{Speedup} \\
\midrule
Cross-Attn. Softmax & Med.-Hard & \xmark & \cmark & \cmark & 2.17$\times$ \\
PreNorm Trans. & Hard & \xmark & \xmark & \cmark & 9.13$\times$ \\
Bound Softmax & Medium & \xmark & \cmark & \cmark & 5.18$\times$ \\
Fused Upsample & Hard & \xmark & \xmark & \cmark & 6.79$\times$ \\
Weight Proc. & Hard & \xmark & \xmark & \cmark & 4.22$\times$ \\
\bottomrule
\end{tabular}
\caption{Multi-turn recovery on medium-to-hard kernels with final speedups. \cmark{} = compile and execution pass; \xmark{} = fail.}
\label{tab:recovery_stats}
\end{table}

These examples demonstrate that RL training enables \textbf{emergent debugging capabilities} on complex GPU kernels: the model learns to handle multi-pass reductions, numerical stability, arbitrary tensor dimensions, and fused operations—patterns that require deep understanding of GPU programming paradigms.

\subsection{Kernel Quality Analysis}
\label{app:kernel_quality}

Across the five medium-to-hard case studies summarized above, the agent-trained model produces correct kernels with speedups of $2.17\times$, $9.13\times$, $5.18\times$, $6.79\times$, and $4.22\times$ over the PyTorch reference (geometric mean $\approx 4.94\times$, median $5.18\times$). Two qualitative observations emerge from inspecting the generated code. First, the optimizations are not generic templates: the model selects different strategies per workload (multi-pass safe softmax for long sequences, fused layer-norm with running statistics for the Pre-Norm decoder, masked tile loads with `tl.exp` numerical stabilization for cross-attention). Second, when a turn fails, the next-turn revision typically targets the exact error class in the feedback (e.g., adding numerical stabilization after a NaN report, switching from single-pass to multi-pass reduction after an out-of-bounds error). This pattern is consistent with the multi-turn agent training objective: the same policy must produce both the failure analysis and the corresponding code edit, so reflections are tied to actionable revisions rather than generic suggestions.

We also note what these examples do \emph{not} show. They are not microbenchmarks tuned for a single shape; each case is validated across multiple shape/stride/dtype combinations under the same test harness used during training, so the reported speedups reflect the realistic correctness-and-performance trade-off our pipeline filters for. A larger-scale latency distribution and memory-consumption study across the full corpus is left for future work and is explicitly listed in our limitations.

\end{document}